\documentclass[journal]{IEEEtran}
\ifCLASSINFOpdf

\else

\fi

\usepackage{picinpar}
\usepackage{url}
\usepackage{flushend}
\usepackage[latin1]{inputenc}
\usepackage{colortbl}
\usepackage{soul}
\usepackage{multirow}
\usepackage{makecell}
\usepackage{booktabs, array}
\usepackage{pifont}
\usepackage{alltt}
\usepackage{fancyhdr}
\usepackage{enumerate}
\usepackage{siunitx}
\usepackage{breakurl}
\usepackage{epstopdf}
\usepackage{pbox}
\usepackage{psfrag}
\usepackage{latexsym, amsthm, subfigure,color, amsfonts, amssymb,graphicx}
\usepackage{algorithm}
\usepackage{algorithmic}
\usepackage{verbatim}
\usepackage[numbers,sort&compress]{natbib}
\usepackage[hidelinks]{hyperref}
\usepackage{xpatch}
\usepackage[justification=centering]{caption}
\usepackage[tbtags]{amsmath}

\usepackage{enumitem}
\setlist[enumerate]{wide=\parindent}
\newtheorem{theorem}{\textbf{Theorem}}
\newtheorem{remark}{\textbf{Remark}}
\newtheorem{definition}{\textbf{Definition}}
\newtheorem{lemma}{\textbf{Lemma}}
\newtheorem{corollary}{\textbf{Corollary}}
\newtheorem{assumption}{\textbf{Assumption}}

\newtheorem{proposition}{\textbf{Proposition}}
\newtheorem{problem}{\textbf{Problem}}
\begin{document}

\title{Privacy-preserving Distributed Machine Learning via Local Randomization and ADMM Perturbation}

\author{Xin Wang,~\IEEEmembership{Student~Member,~IEEE,}
        Hideaki Ishii,~\IEEEmembership{Senior~Member,~IEEE,}
        Linkang Du, ~\IEEEmembership{Student~Member,~IEEE,}
        Peng Cheng,~\IEEEmembership{Member,~IEEE,}
        and~Jiming Chen,~\IEEEmembership{Fellow,~IEEE}% <-this % stops a space
\thanks{X. Wang, L. Du, P. Cheng and J. Chen are with the State Key Lab. of Industrial Control Technology, Zhejiang University, Hangzhou, 310027, P. R. China. X. Wang is also with the Dept. of Computer Science, Tokyo Institute of Technology, Yokohama, 226-8502, Japan. Emails: {\small xinw.zju@gmail.com; linkangd@gmail.com; pcheng@iipc.zju.edu.cn; cjm@zju.edu.cn}}% <-this % stops a space
\thanks{H. Ishii is with the Dept. of Computer Science, Tokyo Institute of Technology, Yokohama, 226-8502, Japan. Email: {\small ishii@c.titech.ac.jp}}%
}

\maketitle

\begin{abstract}
With the proliferation of training data, distributed machine learning (DML) is becoming more competent for large-scale learning tasks. However, privacy concerns have to be given priority in DML, since training data may contain sensitive information of users. In this paper, we propose a privacy-preserving ADMM-based DML framework with two novel features: First, we remove the assumption commonly made in the literature that the users trust the server collecting their data. Second, the framework provides heterogeneous privacy for users depending on data's sensitive levels and servers' trust degrees. The challenging issue is to keep the accumulation of privacy losses over ADMM iterations minimal. In the proposed framework, a local randomization approach, which is differentially private, is adopted to provide users with self-controlled privacy guarantee for the most sensitive information. Further, the ADMM algorithm is perturbed through a combined noise-adding method, which simultaneously preserves privacy for users' less sensitive information and strengthens the privacy protection of the most sensitive information. We provide detailed analyses on the performance of the trained model according to its generalization error. Finally, we conduct extensive experiments using real-world datasets to validate the theoretical results and evaluate the classification performance of the proposed framework.
\end{abstract}


\IEEEpeerreviewmaketitle

\section{Introduction}\label{introduction}
\IEEEPARstart{I}{n} the era of big data, distributed machine learning (DML) is increasingly applied in various areas of our daily lives, especially with proliferation of training data. Typical applications of DML include machine-aided prescription \cite{fredrikson2014privacy}, natural language processing \cite{le2014distributed}, recommender systems \cite{wang2015collaborative}, to name a few. Compared with the traditional single-machine model, DML is more competent for large-scale learning tasks due to its scalability and robustness to faults. The alternating direction method of multipliers (ADMM), as a commonly-used parallel computing approach in optimization community, is a simple but efficient algorithm for multiple servers to collaboratively solve learning problems \cite{boyd2011distributed}. Our DML framework also use ADMM as the underlying algorithm.

However, privacy is a significant issue that has to be considered in DML. In many machine learning tasks, users' data for training the prediction model contains sensitive information, such as genotypes, salaries, and political orientations. For example, if we adopt DML methods to predict HIV-1 infection \cite{qi2010semi}, the data used for protein-protein interactions identification mainly includes patients' information about their proteins, labels indicating whether they are HIV-1 infected or not, and other kinds of health data. Such information, especially the labels, is extremely sensitive for the patients. Moreover, there exist potential risks of privacy disclosure. On one hand, when users report their data to servers, illegal parties can eavesdrop the data transmission processes or penetrate the servers to steal reported data. On the other, the communicated information between servers, which is required to train a common prediction model, can also disclose users' private data. If these disclosure risks are not properly controlled, users would refuse to contribute their data to servers even though DML may bring convenience for them.

Various privacy-preserving solutions have been proposed in the literature. Differential privacy (DP) \cite{dwork2008differential} is one of the standard non-cryptographical approaches and has been applied in distributed computing scenarios \cite{wang2019privacy, nozari2017differentially, dpc2012, wang2018differentially}. Other schemes which are not DP-preserving can be found in \cite{mo2017privacy, manitara2013privacy, he2019consensus}. In addition, privacy-aware machine learning problems \cite{chaudhuri2011differentially, zhang2017dynamic, ding2019optimal, gade2018privacy} have attracted a lot of attentions, and many researchers have proposed ADMM-based solutions \cite{lee2015maximum, zhang2018improving, zhang2019admm}. However, there exists an underlying assumption in most privacy-aware schemes that the data contributors trust the servers collecting their data. This trustworthy assumption may lead to privacy disclosure in many cases. For instance, when the server is penetrated by an adversary, the information obtained by the adversary may be the users' original private data.

Moreover, most existing schemes provide the same privacy guarantee for the entire data sample of a user though different data pieces are likely to have distinct sensitive levels. In the example of HIV-1 infection prediction \cite{qi2010semi} mentioned above, it is obvious that the label indicating HIV-1 infected or uninfected is more sensitive than other health data. Thus, the data pieces with higher sensitive levels should obtain stronger protection. On the other hand, as claimed in \cite{wang2019privacy}, different servers present diverse trust degrees to users due to the distinct permissions to users' data. The servers having no direct connection with a user, compared with the server collecting his/her data, may be less trustworthy. Here, the user would require that the less trustworthy servers obtain his/her information under stronger privacy preservation. Therefore, we will investigate a privacy-aware DML framework that preserves heterogeneous privacy, where users' data pieces with distinct sensitive levels can obtain different privacy guarantee against servers of diverse trust degrees.

One challenging issue is to reduce the accumulation of privacy losses over ADMM iterations as much as possible, especially for the privacy guarantee of the most sensitive data pieces. Most existing ADMM-based private DML frameworks preserve privacy by perturbing the intermediate results shared by servers. Since each intermediate result is computed with users' original data, its release will disclose part of private information, implying that the privacy loss may increase as iterations proceed. Moreover, these private DML frameworks only provide the same privacy guarantee for all data pieces. In addition to intermediate information perturbation, original data randomization methods can be combined to provide heterogeneous privacy protection. However, such an approach introduces coupled uncertainties into the classification model. The lack of uncertainty decoupling methods leads to the performance quantification a challenging task.

In this paper, we propose a privacy-preserving distributed machine learning (PDML) framework to settle these challenging issues. After removing the trustworthy servers assumption, we incorporate the users' data reporting into the DML process, which forms a two-phase training scheme together with the distributed computing process. For privacy preservation, we adopt different approaches in the two phases. In Phase~1, a user first leverages a local randomization approach to obfuscate the most sensitive data pieces and sends the randomized version to a server. This technique provides the user with self-controlled privacy guarantee for the most sensitive information. Further, in Phase~2, multiple servers collaboratively train a common prediction model and there, they use a combined noise-adding method to perturb the communicated messages, which preserves privacy for users' less sensitive data pieces. Also, such perturbation strengthens the privacy preservation of data pieces with the highest sensitive level. For the performance of the PDML framework, we analyze the generalization error of current classifiers trained by different servers.

The main contributions of this paper are threefold:
\begin{enumerate}%[(1)]
\item A two-phase PDML framework is proposed to provide heterogeneous privacy protection in DML, where users' data pieces obtain different privacy guarantees depending on their sensitive levels and servers' trust degrees.
\item In Phase~1, we design a local randomization approach, which preserves DP for the users' most sensitive information. In Phase~2, a combined noise-adding method is devised to compensate the privacy protection of other data pieces.
\item The convergence property of the proposed ADMM-based privacy-aware algorithm is analyzed. We also give a theoretical bound of the difference between the generalization error of trained classifiers and the ideal optimal classifier.
\end{enumerate}

The remainder of this paper is organized as follows. Related works are discussed in Section~\ref{related_works}. We provide some preliminaries and formulate the problem in Section~\ref{pro_formulation}. Section~\ref{pp_framework} presents the designed privacy-preserving framework, and the performance is analyzed in Section~\ref{performance_ana}. In order to validate the classification performance, we use multiple real-world datasets and conduct experiments in Section~\ref{evaluation}. Finally, Section~\ref{conclusion} concludes the paper. A preliminary version \cite{wang2019differential} of this paper was accepted for presentation at IEEE CDC 2019. This paper contains a different privacy-preserving approach with a fully distributed ADMM setting, full proofs of the main results, and more experimental results.
\section{Related Works} \label{related_works}
As one of the important applications of distributed optimization, DML has received widespread attentions from researchers. Besides ADMM schemes, many distributed approaches have been proposed in the literature, e.g., subgradient descent methods \cite{nedic2009distributed}, local message-passing algorithms \cite{predd2009collaborative}, adaptive diffusion mechanisms \cite{chen2012diffusion}, and dual averaging approaches \cite{duchi2011dual}. Compared with these approaches, ADMM schemes achieve faster empirical convergence \cite{shi2014linear}, making it more suitable for large-scale DML tasks.

For privacy-preserving problems, cryptographic techniques \cite{biham2012differential, brakerski2014efficient, shoukry2016privacy} are often used to protect information from being inferred when the key is unknown. In particular, homomorphic encryption methods \cite{brakerski2014efficient}, \cite{shoukry2016privacy} allow untrustworthy servers to calculate with encrypted data, and this approach has been applied in an ADMM scheme \cite{zhang2019admm}. Nevertheless, such schemes unavoidably bring extra computation and communication overheads. Another commonly used approach to preserve privacy is random value perturbation \cite{dwork2008differential}, \cite{erlingsson2014rappor}, \cite{xu2012building}. DP has been increasingly acknowledged as the de facto criterion for non-encryption-based data privacy. This approach requires less costs but still provides strong privacy guarantee, though there exist tradeoffs between privacy and performance \cite{wang2019privacy}.

In recent years, random value perturbation-based approaches have been widely used to address privacy protection in distributed computing, especially in consensus problems \cite{francesco2019lectures}. For instance, \cite{wang2019privacy, nozari2017differentially, dpc2012}, \cite{mo2017privacy, manitara2013privacy, he2019consensus} provide privacy-preserving average consensus paradigms, where the mechanisms in \cite{wang2019privacy, nozari2017differentially, dpc2012} provide DP guarantee. Moreover, for a maximum consensus algorithm, \cite{wang2018differentially} gives a differentially private mechanism. Since these solutions mainly focus on simple statistical analysis (e.g., computation of average and maximum elements), there may exist difficulties in directly applying them to DML.

Privacy-preserving machine learning problems have also attracted a lot of attention recently. Under centralized scenarios, Chaudhuri et al. \cite{chaudhuri2011differentially} proposed a DP solution for an empirical risk minimization problem by perturbing the objective function with well-designed noise. For privacy-aware DML, Han et al. \cite{han2016differentially} also gave a differentially private mechanism, where the underlying distributed approach is subgradient descent. The works \cite{zhang2017dynamic} and \cite{ding2019optimal} present dynamic DP schemes for ADMM-based DML, where privacy guarantee is provided in each iteration. However, if a privacy violator uses the published information in all iterations to make inference, there will be no privacy guarantee. In addition, an obfuscated stochastic gradient method via correlated perturbations was proposed in \cite{gade2018privacy}, though it cannot provide DP preservation. Different from these works, in this paper we remove the trustworthy servers assumption. Moreover, we take into consideration the distinct sensitive levels of data pieces and the diverse trust degrees of servers, and propose the PDML framework providing heterogeneous privacy preservation.

\section{Preliminaries and Problem Statement} \label{pro_formulation}
In this section, we introduce the overall computation framework of DML and the ADMM algorithm used there. Moreover, the privacy-preserving problem for the framework is formulated with the definition of local differential privacy.
\subsection{System Setting}
We consider a collaborative DML framework to carry out classification problems based on data collected from a large number of users. Fig. \ref{framework} gives a schematic diagram. There are two parties involved: Users (or data contributors) and computing servers. The DML's goal is to train a classification model based on data of all users. It has two phases of data collection and distributed computation, called Phase~1 and Phase~2, respectively. In Phase~1, each user sends his/her data to the server, which is responsible to collect all the data from the user's group. In Phase~2, each computing server utilizes a distributed computing approach to cooperatively train the classifier through information interaction with other servers. The proposed DML framework is based on the one in \cite{wang2019privacy}, but the learning tasks are much more complex than the basic statistical analysis considered by \cite{wang2019privacy}.

\textbf{Network Model}. Consider $n\geq2$ computing servers participating in the framework where the $i$th server is denoted by $s_i$. We use an undirected and connected graph $\mathcal{G}=(\mathcal{S}, \mathcal{E})$ to describe the underlying communication topology, where $\mathcal{S}=\{s_i\;| \;i=1, 2, \ldots, n\}$ is the servers set and $\mathcal{E}\subseteq\mathcal{S}\times\mathcal{S}$ is the set of communication links between servers. The number of communication links in $\mathcal{G}$ is denoted by $E$, i.e., $E=|\mathcal{E}|$. Let the set of neighbor servers of $s_i$ be $\mathcal{N}_i=\{s_l\in \mathcal{S}\;| \;(s_i, s_l)\in \mathcal{E}\}$. The degree of server $s_i$ is denoted by $N_i=|\mathcal{N}_i|$.

Different servers collect data from different groups of users, and thus all users can be divided into $n$ distinct groups. The $i$th group of users, whose data is collected by server $s_i$, is denoted by the set $\mathcal{U}_i$, and $m_i=|\mathcal{U}_i|$ is the number of users in $\mathcal{U}_i$. Each user $j\in \mathcal{U}_i$ has a data sample $\mathbf{d}_{i, j}=(\mathbf{x}_{i, j}, y_{i, j})\in \mathcal{X}\times\mathcal{Y} \subseteq \mathbb{R}^{d+1}$, which is composed of a feature vector $\mathbf{x}_{i, j}\in\mathcal{X}\subseteq\mathbb{R}^d$ and the corresponding label $y_{i, j}\in\mathcal{Y} \subseteq \mathbb{R}$. In this paper, we consider a binary-classification problem. That is, there are two types of labels as $y_{i, j}\in\{-1,1\}$. Suppose that all data samples $\mathbf{d}_{i, j}, \forall i,j$, are drawn from an underlying distribution $\mathcal{P}$, which is unknown to the servers. Here, the learning goal is that the classifier trained with limited data samples can match the ideal model trained with known $\mathcal{P}$ as much as possible.
\begin{figure}
  \centering
  \includegraphics[scale=0.5]{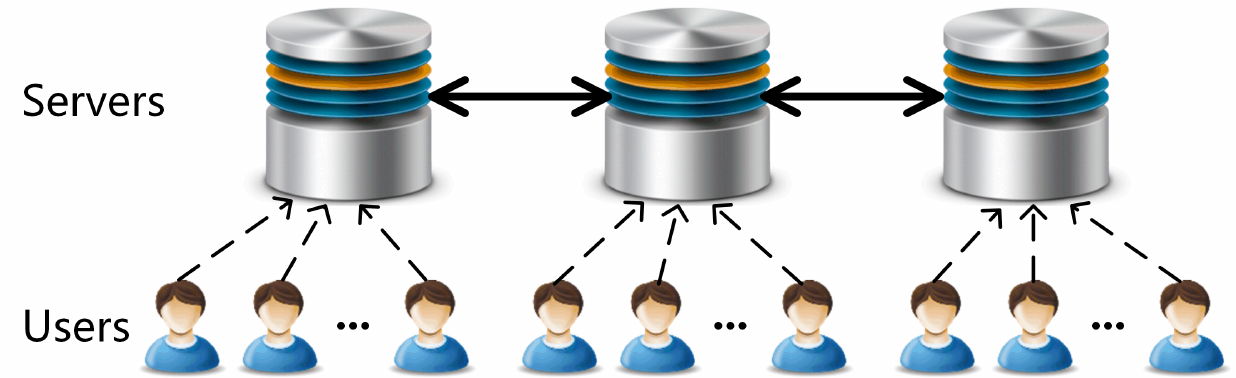}
  \caption{{\small Illustration of the DML framework}}\label{framework}
\end{figure}

\subsection{Classification Problem and ADMM Algorithm} \label{admm_alg}
We first introduce the classification problem solved by the two-phase DML framework. Let $\mathbf{w}: \mathcal{X}\rightarrow\mathcal{Y}$ be the trained classification model. The trained classifier $\mathbf{w}$ should guarantee that the accuracy of mapping any feature vector $\mathbf{x}_{i, j}$ (sampled from the distribution $\mathcal{P}$) to its correct label $y_{i, j}$ is high. We employ the method of regularized empirical risk minimization, which is a commonly used approach to find an appropriate classifier \cite{vapnik2013nature}. Denote the classifier trained by server $s_i$ as $\mathbf{w}_i\in\mathbb{R}^d$. The objective function (or the empirical risk) of the minimization problem is defined as
\begin{equation}\label{original_objective}
\small
J(\{\mathbf{w}_i\}_{i\in\mathcal{S}}):=\sum_{i=1}^n\left[\sum_{j=1}^{m_i}\frac{1}{m_i}\ell(y_{i, j}, \mathbf{w}_i^\mathrm{T}\mathbf{x}_{i, j})+\frac{a}{n}N(\mathbf{w}_i)\right],
\end{equation}
%\end{small}
where $\ell: \mathbb{R}\times\mathbb{R}\rightarrow \mathbb{R}$ is the loss function measuring the performance of the trained classifier $\mathbf{w}_i$. The regularizer $N(\mathbf{w}_i)$ is introduced to mitigate overfitting, and $a>0$ is a constant. We take a bounded classifier class $\mathcal{W}\subset\mathbb{R}^d$ such that $\mathbf{w}_i\in\mathcal{W}, \forall i$. For the loss function $\ell(\cdot)$ and the regularizer $N(\cdot)$, we introduce the following assumptions \cite{chaudhuri2011differentially} \cite{zhang2017dynamic}.%all solved classifiers $\mathbf{w}_i$ should fall within $\mathcal{W}$, i.e., %(whether $\mathbf{w}_i$ maps the feature vectors of available training data to corresponding labels).
\begin{assumption}\label{loss_assumption}
The loss function $\ell(\cdot)$ is convex and doubly differentiable in $\mathbf{w}$. In particular, $\ell(\cdot)$, $\frac{\partial\ell(\cdot)}{\partial \mathbf{w}}$ and $\frac{\partial\ell^2(\cdot)}{\partial \mathbf{w}^2}$ are bounded over the class $\mathcal{W}$ as
\begin{equation}\nonumber
\small
|\ell(\cdot)|\leq c_1, \left|\frac{\partial\ell(\cdot)}{\partial \mathbf{w}}\right|\leq c_2, \left|\frac{\partial\ell^2(\cdot)}{\partial \mathbf{w}^2}\right|\leq c_3,
\end{equation}
where $c_1$, $c_2$ and $c_3$ are positive constants. Moreover, it holds $\frac{\partial\ell^2(y,\mathbf{w}^\mathrm{T}\mathbf{x})}{\partial {\mathbf{w}}^2}=\frac{\partial\ell^2(-y,\mathbf{w}^\mathrm{T}\mathbf{x})}{\partial {\mathbf{w}}^2}$.
\end{assumption}

\begin{assumption}\label{regularizer_assumption}
The regularizer $N(\cdot)$ is doubly differentiable and strongly convex with parameter $\kappa>0$, i.e., $\forall \mathbf{w}_1, \mathbf{w}_2 \in \mathcal{W}$,
\vspace{-0.4cm}
\begin{equation}\label{strongly_convex}
\small
N(\mathbf{w}_1)-N(\mathbf{w}_2)\geq \nabla N(\mathbf{w}_1)^\mathrm{T}(\mathbf{w}_2-\mathbf{w}_1)+\frac{\kappa}{2}\|\mathbf{w}_2-\mathbf{w}_1\|_2^2,
\end{equation}
\vspace{-0.2cm}
where $\nabla N(\cdot)$ indicates the gradient with respect to $\mathbf{w}$.
\end{assumption}

We note that $J(\{\mathbf{w}_i\}_{i\in\mathcal{S}})$ in (\ref{original_objective}) can be separated into $n$ different parts, where each part is the objective function of the local minimization problem to be solved by each server. The objective function of server $s_i$ is
\begin{equation}\label{local_objective}
\small
J_i(\mathbf{w}_i):=\sum_{j=1}^{m_i}\frac{1}{m_i}\ell(y_{i, j}, \mathbf{w}_i^\mathrm{T}\mathbf{x}_{i, j})+\frac{a}{n}N(\mathbf{w}_i).
\end{equation}

Since $\mathbf{w}_i$ is trained based on the data of the $i$th group of users, it may only partially reflect data characteristics. To find a common classifier taking account of all participating users, we place a global consensus constraint in the minimization problem, as $\mathbf{w}_i=\mathbf{w}_l, \forall s_i, s_l\in\mathcal{S}$. However, since we use a connected graph to describe the interaction between servers, we have to utilize a local consensus constraint:
\begin{equation}\label{constraint}
\small
\mathbf{w}_i=\mathbf{z}_{il}, \quad \mathbf{w}_l=\mathbf{z}_{il}, \quad \forall (s_i, s_l)\in \mathcal{E},
\end{equation}
where $\mathbf{z}_{il}\in\mathbb{R}^d$ is an auxiliary variable enforcing consensus between neighbor servers $s_i$ and $s_l$. Obviously, (\ref{constraint}) also implies global consensus. We can now write the whole regularized empirical risk minimization problem as follows \cite{forero2010consensus}.
\begin{problem} \label{problem_1}
\begin{small}
\begin{alignat}{2}
    \min_{\{\mathbf{w}_i\}, \{\mathbf{z}_{i, l}\}} & \sum_{i=1}^n\left[\sum_{j=1}^{m_i}\frac{1}{m_i}\ell(y_{i, j}, \mathbf{w}_i^\mathrm{T}\mathbf{x}_{i, j})+\frac{a}{n}N(\mathbf{w}_i)\right] \label{minimization_pro} \\
    \mathrm{s.t.} \quad & \mathbf{w}_i=\mathbf{z}_{il}, \quad \mathbf{w}_l=\mathbf{z}_{il}, \quad \forall (s_i, s_l)\in \mathcal{E}.
\end{alignat}
\end{small}
\end{problem}

Next, we establish a compact form of Problem~\ref{problem_1}. Let $\mathbf{w}:=[\mathbf{w}_1^\mathrm{T} \cdots \mathbf{w}_n^\mathrm{T}]^\mathrm{T}\in\mathbb{R}^{nd}$ and $\mathbf{z}\in\mathbb{R}^{2Ed}$ be vectors aggregating all classifiers $\mathbf{w}_i$ and auxiliary variables $\mathbf{z}_{il}$, respectively. To transfer all local consensus constraints into a matrix form, we introduce two block matrices $A_1, A_2\in \mathbb{R}^{{2Ed}\times{nd}}$, which are partitioned into $2E\times n$ submatrices with dimension $d\times d$. For the communication link $(s_i, s_l)\in\mathcal{E}$, if $\mathbf{z}_{il}$ is the $m$th block of $\mathbf{z}$, then the $(m, i)$th submatrix of $A_1$ and $(m, l)$th submatrix of $A_2$ are the $d\times d$ identity matrix $I_d$; otherwise, these submatrices are the $d\times d$ zero matrix $0_d$. We write $J(\mathbf{w})=\sum_{i=1}^n J(\mathbf{w}_i)$, $A:=[A_1^{\mathrm{T}} A_2^{\mathrm{T}}]^{\mathrm{T}}$, and $B:=[-I_{2Ed}\; -\!I_{2Ed}]^{\mathrm{T}}$. Then, Problem~\ref{problem_1} can be written in a compact form as
\begin{alignat}{2}
    \min_{\mathbf{w}, \mathbf{z}} \quad & J(\mathbf{w}) \label{minimization_pro_mat} \\
    \mathrm{s.t.} \quad & A\mathbf{w}+B\mathbf{z}=0. \label{constraint_mat}
\end{alignat}

For solving this problem we introduce the fully distributed ADMM algorithm from \cite{shi2014linear}. The augmented Lagrange function associated with (\ref{minimization_pro_mat}) and (\ref{constraint_mat}) is given by $\mathcal{L}(\mathbf{w}, \mathbf{z}, \boldsymbol{\lambda}) := J(\mathbf{w}) +\boldsymbol{\lambda}^{\mathrm{T}}(A\mathbf{w}+B\mathbf{z})+ \frac{\beta}{2}\|A\mathbf{w}+B\mathbf{z}\|_2^2$, where $\boldsymbol{\lambda}\in\mathbb{R}^{4Ed}$ is the dual variable ($\mathbf{w}$ is correspondingly called the primal variable) and $\beta\in \mathbb{R}$ is the penalty parameter.

At iteration $t+1$, the solved optimal auxiliary variable $\mathbf{z}$ satisfies the relation $\nabla \mathcal{L}(\mathbf{w}(t+1), \mathbf{z}(t+1), \boldsymbol{\lambda}(t))=0$. Through some simple transformation, we have $B^\mathrm{T}\boldsymbol{\lambda}(t+1)=0$. Let $\boldsymbol{\lambda}=[\boldsymbol{\xi}^{\mathrm{T}} \boldsymbol{\zeta}^{\mathrm{T}}]^{\mathrm{T}}$ with $\boldsymbol{\xi}, \boldsymbol{\zeta}\in\mathbb{R}^{2Ed}$. If we set the initial value of $\boldsymbol{\lambda}$ to $\boldsymbol{\xi}(0)=-\boldsymbol{\zeta}(0)$, we have $\boldsymbol{\xi}(t)=-\boldsymbol{\zeta}(t), \forall t\geq 0$. Thus, we can obtain the complete dual variable $\boldsymbol{\lambda}$ by solving $\boldsymbol{\xi}$. Let
\begin{equation}\nonumber
\small
L_{+} := \frac{1}{2}(A_1+A_2)^\mathrm{T}(A_1+A_2),
L_{-} := \frac{1}{2}(A_1-A_2)^\mathrm{T}(A_1-A_2).
\end{equation}
Define a new dual variable $\boldsymbol{\gamma}:=(A_1-A_2)^\mathrm{T}\boldsymbol{\xi}\in\mathbb{R}^{nd}$. Through the simplification process in \cite{shi2014linear}, we obtain the fully distributed ADMM for solving Problem~\ref{problem_1}, which is composed of the following iterations:
\begin{alignat}{2}
     \small \nabla J(\mathbf{w}(t+1)) +\boldsymbol{\gamma}(t) + \beta(L_{+}+L_{-})\mathbf{w}(t+1) -  \beta L_{+}\mathbf{w}(t)& = 0, \nonumber \\
     \small \boldsymbol{\gamma}(t+1) - \boldsymbol{\gamma}(t) - \beta L_{-}\mathbf{w}(t+1)& = 0. \nonumber
\end{alignat}
Note that $\boldsymbol{\gamma}$ is also a compact vector of all local dual variables $\boldsymbol{\gamma}_i\in\mathbb{R}^{d}$ for $s_i\in\mathcal{S}$, i.e., $\boldsymbol{\gamma}=[\boldsymbol{\gamma}_1^\mathrm{T} \cdots \boldsymbol{\gamma}_n^\mathrm{T}]^\mathrm{T}$.

The above ADMM iterations can be separated into $n$ different parts, which are solved by the $n$ different servers. At iteration $t+1$, the information used by server $s_i$ to update a new primal variable $\mathbf{w}_i(t+1)$ includes users' data $\mathbf{d}_{i,j}, \forall j$, current classifiers $\left\{\mathbf{w}_l(t)\;|\; l\in\mathcal{N}_i\bigcup\{i\}\right\}$ and dual variable $\boldsymbol{\gamma}_i(t)$. The local augmented Lagrange function $\mathcal{L}_i(\mathbf{w}_i, \mathbf{w}_i(t), \boldsymbol{\gamma}_i(t))$ associated with the primal variable update is given by
\begin{equation}\nonumber
\small
\begin{split}
&  \mathcal{L}_i(\mathbf{w}_i, \{\mathbf{w}_l(t)\}_{l\in \mathcal{N}_i\bigcup\{i\}},\boldsymbol{\gamma}_i(t)) \\
& :=J_i(\mathbf{w}_i) +\boldsymbol{\gamma}_i^{\mathrm{T}}(t)\mathbf{w}_i
+\beta\sum_{l\in\mathcal{N}_i}\left\|\mathbf{w}_i-\frac{1}{2}(\mathbf{w}_i(t)+\mathbf{w}_l(t))\right\|_2^2.
\end{split}
\end{equation}
At each iteration, server $s_i$ will update its primal variable $\mathbf{w}_i(t+1)$ and dual variable $\boldsymbol{\gamma}_i(t+1)$ as follows:
\begin{alignat}{2}
     \small \mathbf{w}_i(t+1)& = \arg\min_{\mathbf{w}_i} \mathcal{L}_i(\mathbf{w}_i, \{\mathbf{w}_l(t)\}_{l\in \mathcal{N}_i\bigcup\{i\}}, \boldsymbol{\gamma}_i(t)), \label{new_primal_local} \\
     \small \boldsymbol{\gamma}_i(t+1)& =  \boldsymbol{\gamma}_i(t) + \beta \sum_{l\in\mathcal{N}_i}\left(\mathbf{w}_i(t+1) - \mathbf{w}_l(t+1)\right). \label{new_dual_local}
\end{alignat}
Clearly, in (\ref{new_primal_local}) and (\ref{new_dual_local}), the information communicated between computing servers is the newly updated classifiers.

\subsection{Privacy-preserving Problem}
In this subsection, we introduce the privacy-preserving problem in the DML framework. The private information to be preserved is first defined, followed by the introduction of privacy violators and information used for privacy inference. Further, we present the objectives of the two phases.

\textbf{Private information}. For users, both the feature vectors and the labels of the data samples contain their sensitive information. The private information contained in the feature vectors may be the ID, gender, general health data and so on. However, the labels may indicate, for example, whether a patient contracts a disease (e.g., HIV-1 infected) or whether a user has a special identity (e.g., a member of a certain group). We can see that compared with the feature vectors, the labels may be more sensitive for the users. In this paper, we consider that the labels of users' data are the most sensitive information, which should be protected with priority and obtain stronger privacy guarantee than that of feature vectors.

\textbf{Privacy attacks}. All computing servers are viewed as untrustworthy potential privacy violators desiring to infer the sensitive information contained in users' data. In the meantime, different servers present distinct trust degrees to users. User $j\in\mathcal{U}_i$ divides the potential privacy violators into two types. The server $s_i$, collecting user $j$'s data directly, is the first type. Other servers $s_l\in\mathcal{S}, s_l\neq s_i$, having no direct connection with user $j$, are the second type. Compared with server $s_i$, other servers are less trustworthy for user $j$. To conduct privacy inference, the first type of privacy violators leverages user $j$'s reported data while the second type can utilize only the intermediate information shared by servers.

\textbf{Privacy protections in Phases 1 \& 2}. Since the label of user $j\in\mathcal{U}_i$ is the most sensitive information, its original value should not be disclosed to any servers including server $s_i$. Thus, during the data reporting process in Phase~1, user $j$ must obfuscate the private label in his/her local device. For the less sensitive feature vector, considering that server $s_i$ is more trustworthy, user $j$ can choose to transmit the original version to that server. Nevertheless, the user is still unwilling to disclose the raw feature vector to servers with lower trust degrees. Hence, in this paper, when server $s_i$ interacts with other servers to find a common classifier in Phase~2, the released information about user $j$'s data will be further processed before communication. %On one hand, such information processing is necessary to avoid direct privacy disclosure of users' feature vectors to the less trustworthy servers. On the other hand, further privacy-preserving operation on the obfuscated labels makes it more difficult for the less trustworthy servers to identify the original sensitive labels, which implies stronger privacy guarantee for users' labels in Phase 2.

More specifically, in Phase 1, to obfuscate the labels, we use a local randomization approach, whose privacy-preserving property will be measured by local differential privacy (LDP) \cite{erlingsson2014rappor}. LDP is developed from differential privacy (DP), which is originally defined for trustworthy databases to publish aggregated private information \cite{dwork2008differential}. The privacy preservation idea of DP is that for any two neighbor databases differing in one record (e.g., one user selects to report or not to report his/her data to the server) as input, a randomized mechanism is adopted to guarantee the two outputs to have high similarity so that privacy violators cannot identify the different record with high confidence. Since there is no trusted server for data collection in our setting, users locally perturb their original labels and report noisy versions to the servers.

To this end, we define a randomized mechanism $M: \mathbb{R}^{d+1}\rightarrow \mathbb{R}^{d+1}$, which takes a data sample as input and outputs its noisy version. The definition of LDP is given as follows.
\begin{definition}\label{LDP_def}
($\epsilon$-LDP). Given $\epsilon>0$, a randomized mechanism $M(\cdot)$ preserves $\epsilon$-LDP if for any two data samples $\mathbf{d}_1=(\mathbf{x}_1, y_1)$ and $\mathbf{d}_2=(\mathbf{x}_2, y_2)$ satisfying $\mathbf{x}_2=\mathbf{x}_1$ and $y_2=-y_1$, and any observation set $\mathcal{O}\subseteq \textrm{Range}(M)$, it holds
\begin{equation}\label{eq_ldp}
\small
\Pr[M(\mathbf{d}_1)\in \mathcal{O}] \leq e^{\epsilon}\Pr[M(\mathbf{d}_2)\in \mathcal{O}].
\end{equation}
\end{definition}

In (\ref{eq_ldp}), the parameter $\epsilon$ is called the privacy preserving degree (PPD), which describes the strength of privacy guarantee of $M(\cdot)$. A smaller $\epsilon$ implies stronger privacy guarantee. This is because smaller $\epsilon$ means that the two outputs $M(\mathbf{d}_1)$ and $M(\mathbf{d}_2)$ are closer, making it more difficult for privacy violators to infer the difference in $\mathbf{d}_1$ and $\mathbf{d}_2$ (i.e., $y_1$ and $y_2$).

\subsection{System Overview}
In this paper, we propose the PDML framework, where users can obtain heterogeneous privacy protection. The heterogeneity is characterized by two aspects: i) When a user faces a privacy violator, his/her data pieces with distinct sensitive levels (i.e., the feature vector and the label) obtain different privacy guarantees; ii) for one type of private data piece, the privacy protection provided by the framework is stronger against privacy violators with low trust degrees than those with higher trust degrees. Particularly, in our approach, the privacy preservation strength of users' labels is controlled by the users. Moreover, a modified ADMM algorithm is proposed to meet the heterogeneous privacy protection requirement.

The workflow of the proposed PDML framework is illustrated in Fig. \ref{workflow}. Some details are explained below.
\begin{enumerate}
  \item In Phase 1, a user first appropriately randomizes the private label, and then sends the noisy label and the original feature vector to a computing server. The randomization approach used here determines the PPD of the label.
  \item In Phase 2, multiple computing servers collaboratively train a common classifier based on their collected data. To protect privacy of feature vectors against less trustworthy servers, we further use a combined noise-adding method to perturb the ADMM algorithm, which also strengthens the privacy guarantee of users' labels.
  \item The performance of the trained classifiers is analyzed in terms of their generalization errors. To decompose the effects of uncertainties introduced in the two phases, we modify the loss function in Problem~\ref{problem_1}. We finally quantify the difference between the generalization error of trained classifiers and that of the ideal optimal classifier.
\end{enumerate}
\begin{figure}
  \centering
  \includegraphics[scale=0.5]{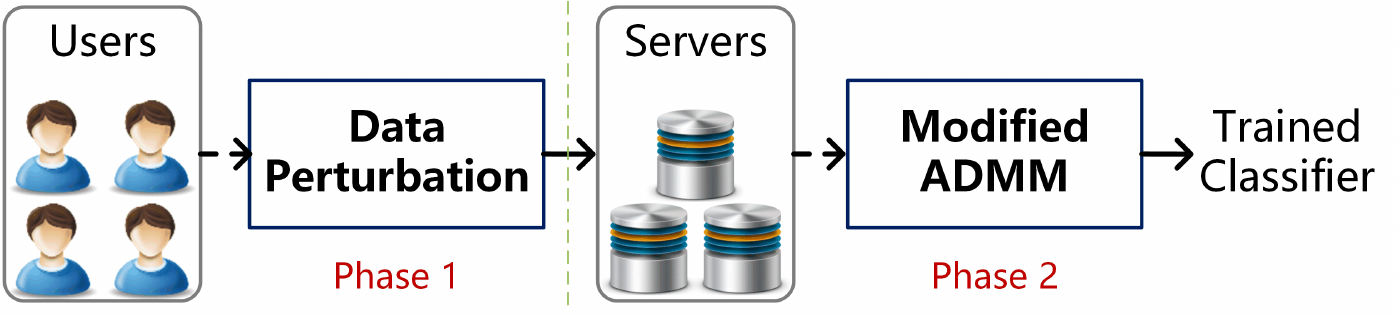}
  \caption{{\small Workflow of the PDML framework}}\label{workflow}
\end{figure}
\section{Privacy-Preserving Framework Design} \label{pp_framework}
In this section, we introduce the privacy-preserving approaches used in Phases 1 and 2, and analyze their properties.
\subsection{Privacy-Preserving Approach in Phase 1}
In this subsection, we propose an appropriate approach used in Phase~1 to provide privacy preservation for the most sensitive labels. In particular, it is controlled by users and will not be weakened in Phase~2.

We adopt the idea of randomized response (RR) \cite{erlingsson2014rappor} to obfuscate the users' labels. Originally, RR was used to set plausible deniability for respondents when they answer survey questions about sensitive topics (e.g., HIV-1 infected or uninfected). When using RR, respondents only have a certain probability to answer questions according to their true situations, making the server unable to determine with certainty whether the reported answers are true.

In our setting, user $j\in\mathcal{U}_i$ randomizes the label through RR and sends the noisy version to server $s_i$. This is done by the randomized mechanism $M$ defined below.
\begin{definition}\label{randomized_M}
For $p\in(0,\frac{1}{2})$, the randomized mechanism $M$ with input data sample $\mathbf{d}_{i, j}=(\mathbf{x}_{i, j}, y_{i, j})$ is given by $M(\mathbf{d}_{i, j})=(\mathbf{x}_{i, j}, y'_{i, j})$, where
% {-0.3cm}
\begin{equation} \label{randomization}
\small
y'_{i, j}=
\begin{cases}
1, &\text{with probability $p$} \\
-1, &\text{with probability $p$} \\
y_{i, j}, &\text{with probability $1-2p$}.
\end{cases}
\end{equation}
\end{definition}

In the above definition, $p$ is the randomization probability controlling the level of data obfuscation. Obviously, a larger $p$ implies higher uncertainty on the reported label, making it harder for the server to learn the true label.

Denote the output $M(\mathbf{d}_{i, j})$ as $\mathbf{d}'_{i, j}$, i.e., $\mathbf{d}'_{i, j}=M(\mathbf{d}_{i, j})=(\mathbf{x}_{i, j}, y'_{i, j})$. After the randomization, $\mathbf{d}'_{i, j}$ will be transmitted to the server. In this case, server $s_i$ can use only $\mathbf{d}'_{i, j}$ to train the classifier, and the released information about the true label $y_{i, j}$ in Phase~2 is computed based on $\mathbf{d}'_{i, j}$. This implies that once $\mathbf{d}'_{i, j}$ is reported to the server, no more information about the true label $y_{i, j}$ will be released. In this paper, we set the randomization probability $p$ in (\ref{randomization}) as
\begin{equation}\label{eq_p}
\small
p=\frac{1}{1+e^\epsilon},
\end{equation}
where $\epsilon>0$. The following theorem gives the privacy-preserving property of the randomized mechanism in Definition~\ref{randomized_M}, justifying this choice of $p$ from the viewpoint of LDP.
\begin{proposition}\label{privacy_preservation}
Under (\ref{eq_p}), the randomized mechanism $M(\mathbf{d}_{i, j})$ preserves $\epsilon$-LDP for $\mathbf{d}_{i, j}$.
\end{proposition}
The proof can be found in Appendix \ref{proof_p1}.

Proposition \ref{privacy_preservation} clearly indicates that the users can tune the randomization probability according to their privacy demands. This can be seen as given a randomization probability $p$, by (\ref{eq_p}), the PPD $\epsilon$ provided by $M(\mathbf{d}_{i, j})$ is $\epsilon=\ln \frac{1-p}{p}$. Obviously, a larger randomization probability leads to smaller PPD, indicating stronger privacy guarantee.

If all data samples $\mathbf{d}_{i, j}, \forall i, j$, drawn from the distribution $\mathcal{P}$ are randomized through $M$, the noisy data $\mathbf{d}'_{i, j}, \forall i, j$, can be considered to be obtained from a new distribution $\mathcal{P}_{\epsilon}$, which is related to the PPD $\epsilon$. Note that $\mathcal{P}_{\epsilon}$ is also an unknown distribution due to the unknown $\mathcal{P}$.
\subsection{Privacy-Preserving Approach in Phase 2} \label{ppa_2}
To deal with less trustworthy servers in Phase~2, we devise a combined noise-adding approach to simultaneously preserve privacy for users' feature vectors and enhance the privacy guarantee of users' labels. We first adopt the method of objective function perturbation \cite{chaudhuri2011differentially}. That is, before solving Problem \ref{problem_1}, the servers perturb the objective function $J(\{\mathbf{w}_i\}_{i\in\mathcal{S}})$ with random noises. For server $s_i\in \mathcal{S}$, the perturbed objective function is given by %Since the modified loss function still contains users' original fracture vectors and noisy labels, the classier $\mathbf{w}_i(t)$ obtained by solving Problem \ref{problem_2}, i.e., the communicated messages between servers, may disclose private information of users' data.
% {-0.3cm}
\begin{equation}\label{per_obj}
\small
\widetilde{J}_i(\mathbf{w}_i):=J_i(\mathbf{w}_i)+\frac{1}{n}\boldsymbol{\eta}_i^{\mathrm{T}}\mathbf{w}_i,
% {-0.2cm}
\end{equation}
where $J_i(\mathbf{w}_i)$ is the local objective function given in (\ref{local_objective}), and $\boldsymbol{\eta}_i\in\mathbb{R}^{d}$ is a bounded random noise with arbitrary distribution. Let $R$ be the bound of noises $\boldsymbol{\eta}_i, \forall i$, namely, $\|\boldsymbol{\eta}_i\|_{\infty}\leq R$. Denote the sum of $\widetilde{J}_i(\mathbf{w}_i)$ as $\widetilde{J}(\{\mathbf{w}_i\}_{i\in\mathcal{S}}):=\sum_{i=1}^{n} \widetilde{J}_i(\mathbf{w}_i)$. %We let all servers perturb their objective functions using (\ref{per_obj}), and thus the optimization problem becomes as follows:
%\begin{problem} \label{problem_3}
%\begin{small}
%\begin{alignat}{2}
%    \min_{\{\mathbf{w}_i\}} &\quad  \widetilde{J}(\{\mathbf{w}_i\}_{i\in\mathcal{S}}) \nonumber \\
%    \mathrm{s.t.} &\quad  \mathbf{w}_i= \mathbf{w}_l, \forall i, l. \nonumber
%\end{alignat}
%\end{small}
%\end{problem}
%In Problem \ref{problem_3}, we use the global consensus constraint to replace the local consensus constraint in (\ref{constraint}). This is because the two constraints are equivalent.%, and the global consensus constraint is more convenient for the denotation of the solutions.

\textbf{Limitation of objective function perturbation}. We remark that in our setting, the objective function perturbation in (\ref{per_obj}) is not sufficient to provide reliable privacy guarantee. This is because each server publishes current classifier multiple times and each publication utilizes users' reported data. Note that in the more centralized setting of \cite{chaudhuri2011differentially}, the classifier is only published once. More specifically, according to (\ref{new_primal_local}), $\mathbf{w}_i(t+1)$ is the solution to $\nabla \mathcal{L}_i(\mathbf{w}_i, \{\mathbf{w}_l(t)\}_{l\in \mathcal{N}_i\bigcup\{i\}}, \boldsymbol{\gamma}_i(t))=0$. In this case, it holds $\nabla \widetilde{J}_i(\mathbf{w}_i(t+1))= -\boldsymbol{\gamma}_i(t)
+\beta\sum_{l\in\mathcal{N}_i}\left(\mathbf{w}_i(t)+\mathbf{w}_l(t)-2\mathbf{w}_i(t+1)\right)$. As (\ref{new_dual_local}) shows, the dual variable $\boldsymbol{\gamma}_i(t)$ can be deduced from updated classifiers. Thus, if $s_i$'s neighbor servers have access to $\mathbf{w}_i(t+1)$ and $\left\{\mathbf{w}_l(t)\;|\;l\in\{i\}\bigcup \mathcal{N}_i\right\}$, then they can easily compute $\nabla \widetilde{J}_i(\mathbf{w}_i(t+1))$.

We should highlight that multiple releases of $\nabla \widetilde{J}_i(\mathbf{w}_i(t+1))$ increase the risk of users' privacy disclosure. This can be explained as follows. First, note that $\nabla \widetilde{J}_i(\mathbf{w}_i)=\nabla J_i(\mathbf{w}_i)+\frac{1}{n}\boldsymbol{\eta}_i$, where $\nabla J_i(\mathbf{w}_i)$ contains users' private information. The goal of $\boldsymbol{\eta}_i$-perturbation is to protect $\nabla J_i(\mathbf{w}_i)$ not to be derived directly by other servers. However, after publishing an updated classifier $\mathbf{w}_i(t+1)$, server $s_i$ releases a new gradient $\nabla \widetilde{J}_i(\cdot)$. Since the noise $\boldsymbol{\eta}_i$ is fixed for all iterations, each release of $\nabla \widetilde{J}_i(\cdot)$ means disclosing more information about $\nabla J_i(\cdot)$. In particular, we have $\nabla \widetilde{J}_i(\mathbf{w}_i(t+1))-\nabla \widetilde{J}_i(\mathbf{w}_i(t))=\nabla J_i(\mathbf{w}_i(t+1))-\nabla J_i(\mathbf{w}_i(t))$. That is, the effect of the added noise $\boldsymbol{\eta}_i$ can be cancelled by integrating the gradients of objective functions at different time instants.

\textbf{Modified ADMM by primal variable perturbation}. To ensure appropriate privacy preservation in Phase~2, we adopt an extra perturbation method, which sets obstructions for other servers to obtain the gradient $\nabla J_i(\cdot)$. Specifically, after deriving classifier $\mathbf{w}_i(t)$, server $s_i$ first perturbs $\mathbf{w}_i(t)$ with a Gaussian noise $\boldsymbol{\theta}_i(t)$ whose variance is decaying as iterations proceed, and then sends a noisy version of $\mathbf{w}_i(t)$ to neighbor servers. This is denoted by $\widetilde{\mathbf{w}}_i(t):=\mathbf{w}_i(t) + \boldsymbol{\theta}_i(t)$, where $\boldsymbol{\theta}_i(t)\sim\mathcal{N}(0, \rho^{t-1}V_i^2I_d)$ with decaying rate $0<\rho<1$.

The local augmented Lagrange function associated with $\boldsymbol{\eta}_i$-perturbed objective function $\widetilde{J}_i(\mathbf{w}_i)$ in (\ref{per_obj}) is given by
\begin{equation}\nonumber
\small
\begin{split}
& \tilde{\mathcal{L}}_i(\mathbf{w}_i, \{\widetilde{\mathbf{w}}_l(t)\}_{l\in \mathcal{N}_i\bigcup\{i\}},\boldsymbol{\gamma}_i(t)) \\
& :=\widetilde{J}_i(\mathbf{w}_i)  +\boldsymbol{\gamma}_i^{\mathrm{T}}(t)\mathbf{w}_i
+\beta\sum_{l\in\mathcal{N}_i}\left\|\mathbf{w}_i-\frac{1}{2}(\widetilde{\mathbf{w}}_i(t)+\widetilde{\mathbf{w}}_l(t))\right\|_2^2.
\end{split}
\end{equation}

We then introduce the perturbed version of the ADMM algorithm in (\ref{new_primal_local}) and (\ref{new_dual_local}) as
\begin{alignat}{2}
     \small  \mathbf{w}_i(t+1)& = \arg\min_{\mathbf{w}_i} \tilde{\mathcal{L}}_i(\mathbf{w}_i, \{\widetilde{\mathbf{w}}_l(t)\}_{l\in \mathcal{N}_i\bigcup\{i\}}, \boldsymbol{\gamma}_i(t)), \label{wi_update} \\
     \small \widetilde{\mathbf{w}}_i(t+1)& = \mathbf{w}_i(t+1) + \boldsymbol{\theta}_i(t+1), \label{wi_perturb}\\
     \small \boldsymbol{\gamma}_i(t+1)& =  \boldsymbol{\gamma}_i(t) + \beta \sum_{l\in\mathcal{N}_i}\left(\widetilde{\mathbf{w}}_i(t+1) - \widetilde{\mathbf{w}}_l(t+1)\right). \label{gammai_update}
\end{alignat}
At iteration $t+1$, a new classifier $\mathbf{w}_i(t+1)$ is first obtained by solving $\nabla \tilde{\mathcal{L}}_i(\mathbf{w}_i, \widetilde{\mathbf{w}}_i(t), \boldsymbol{\gamma}_i(t))=0$. Then, server $s_i$ will send $\widetilde{\mathbf{w}}_i(t+1)$ out and wait for the updated classifiers from neighbor servers.  At the end of an iteration, the server will update the dual variable $\boldsymbol{\gamma}_i(t+1)$.
\subsection{Discussions} \label{discussion_privacy}
We now discuss the effectiveness of the primal variable perturbation. It is emphasized that at each iteration, $s_i$ only releases a small amount of information about $\nabla \widetilde{J}_i(\mathbf{w}_i(t+1))$ through the communicated $\widetilde{\mathbf{w}}_i(t+1)$. Although $\boldsymbol{\gamma}_i(t)$ and $\left\{\widetilde{\mathbf{w}}_l(t)\;|\;l\in\{i\}\bigcup \mathcal{N}_i\right\}$ are known to $s_i$'s neighbors, $\nabla \widetilde{J}_i(\mathbf{w}_i(t+1))$ cannot be directly computed due to the unknown $\boldsymbol{\theta}_i(t+1)$. More specifically, observe that by (\ref{wi_update}), we have $\nabla \widetilde{J}_i(\mathbf{w}_i(t+1))= -\boldsymbol{\gamma}_i(t)+\beta\sum_{l\in\mathcal{N}_i}\left(\widetilde{\mathbf{w}}_i(t)+\widetilde{\mathbf{w}}_l(t)\right) -2\beta N_i(\widetilde{\mathbf{w}}_i(t+1) -\boldsymbol{\theta}_i(t+1))$, where $N_i$ is the degree of $s_i$.

On the other hand, using available information, other servers can compute only $\nabla \widetilde{J}_i(\widetilde{\mathbf{w}}_i(t+1))$, i.e., the gradient with respect to perturbed classifier $\widetilde{\mathbf{w}}_i(t+1)$. We have $\nabla \widetilde{J}_i(\widetilde{\mathbf{w}}_i(t+1))=-\boldsymbol{\gamma}_i(t)-\beta\sum_{l\in\mathcal{N}_i}\left[2\widetilde{\mathbf{w}}_i(t+1)-(\widetilde{\mathbf{w}}_i(t)+ \widetilde{\mathbf{w}}_l(t))\right]$. Thus, we obtain $\nabla \widetilde{J}_i(\widetilde{\mathbf{w}}_i(t+1))-\nabla \widetilde{J}_i(\widetilde{\mathbf{w}}_i(t))= \nabla J_i(\mathbf{w}_i(t+1))-\nabla J_i(\mathbf{w}_i(t))-2\beta N_i(\boldsymbol{\theta}_i(t+1)-\boldsymbol{\theta}_i(t))$. Hence, due to $\boldsymbol{\theta}_i$, it would not be helpful for inferring $\nabla J_i(\cdot)$ to integrate the gradients of the objective functions at different iterations. %Since $\nabla \widetilde{J}_i(\mathbf{w}_i(t+1))$ cannot be directly deduced and  is not satisfied, it is much more difficult for privacy violators to infer $\nabla \widehat{J}_i(\cdot)$.% (containing users' private information).

We should also observe that since $\lim_{t\rightarrow\infty} \boldsymbol{\theta}_i(t+1)=0$, $\nabla \widetilde{J}_i(\mathbf{w}_i(t+1))$ can be derived when $t\rightarrow\infty$. Moreover, it is clear that the relation $\nabla \widetilde{J}_i(\widetilde{\mathbf{w}}_i(t+1))-\nabla \widetilde{J}_i(\widetilde{\mathbf{w}}_i(t))= \nabla J_i(\mathbf{w}_i(t+1))-\nabla J_i(\mathbf{w}_i(t))$ holds for $t\rightarrow\infty$. However, $\nabla \widetilde{J}_i(\cdot)$ is the result of $\nabla J_i(\cdot)$ under $\boldsymbol{\eta}_i$-perturbation. Moreover, due to the local consensus constraint (\ref{constraint}), the trained classifiers $\mathbf{w}_i(t)$ may not have significant differences when $t\rightarrow\infty$. Such limited information is not sufficient for privacy violators to infer $\nabla J_i(\cdot)$ with high confidence.
%Further, summing the deviations of the current $t$ rounds and taking the limit, we derive
%\begin{equation}\nonumber
%\small
%\begin{split}
% & \lim_{t\rightarrow\infty} \sum_{k=1}^{t} \left[\nabla \widetilde{J}_i(\widetilde{\mathbf{w}}_i(t+1))-\nabla \widetilde{J}_i(\widetilde{\mathbf{w}}_i(t))\right] \\
%& = \lim_{t\rightarrow\infty} \nabla \widehat{J}_i(\mathbf{w}_i(t))
%- \nabla \widehat{J}_i(\mathbf{w}_i(1)) -2\beta N_i \boldsymbol{\theta}_i(1).
%\end{split}
%\end{equation}
%Still, due to the noise $\boldsymbol{\theta}_i(1)$, it does not provide much more information for inferring $\boldsymbol{\eta}_i$ to use the gradients in multiple iterations. Thus, the step of $\mathbf{w}_i$-perturbation in (\ref{wi_perturb}) can keep $\boldsymbol{\eta}_i$ secret to other servers.

\textbf{Differential privacy analysis}. We remark that in our scheme, the noise $\boldsymbol{\eta}_i$ added to the objective function provides underlying privacy protection in Phase~2. Even if privacy violators make inference with published $\widetilde{\mathbf{w}}_i$ in all iterations, the disclosed information is users' reported data plus extra noise perturbation. If the objective function perturbation is removed, the primal variable perturbation method cannot provide DP guarantee when $t\rightarrow\infty$. It is proved in \cite{zhang2017dynamic} and \cite{ding2019optimal} that the $\mathbf{w}_i$-perturbation in (\ref{wi_perturb}) preserves dynamic DP. According to the composition theorem of DP \cite{dwork2008differential}, the PPD will increase (indicating weaker privacy guarantee) when other servers obtain the perturbed classifiers $\widetilde{\mathbf{w}}_i$ of multiple iterations. In particular, if the perturbed classifiers in all iterations are used for inference, the PPD will be $\infty$, implying no privacy guarantee any more.
\begin{remark}
The objective function perturbation given in (\ref{per_obj}) preserves the so-called $(\epsilon_p, \delta_p)$-DP \cite{he2017differential}. Also, according to \cite{chaudhuri2011differentially}, the perturbation in (\ref{per_obj}) preserves $\epsilon_2$-DP if $\boldsymbol{\eta}_i$ has density $f(\boldsymbol{\eta}_i)=\frac{1}{\nu}e^{-\epsilon_2\|\boldsymbol{\eta}_i\|}$ with normalizing parameter $\nu$. Note that the noise with this density is not bounded, which is not consistent with our setting. Although we use a bounded noise, this kind of perturbation still provides $(\epsilon_p, \delta_p)$-DP guarantee, which is a relaxed form of pure $\epsilon_p$-DP.
\end{remark}

\textbf{Strengthened privacy guarantee}. For users' labels, the privacy guarantee in Phase~2 is stronger than that of Phase~1. Since differential privacy is immune to post-processing \cite{dwork2008differential}, the PPD $\epsilon$ in Phase~1 will not increase during the iterations of the ADMM algorithm executed in Phase~2. However, such immunity is established based on a strong assumption that there is no limit to the capability of privacy violators. In our considered problem, this assumption is satisfied when all servers can have access to user $j$'s reported data $\mathbf{d}'_{i, j}$, which may not be realistic. Hence, in our problem setting, one server (i.e., server $s_i$) obtains $\mathbf{d}'_{i, j}$ while other servers can access only the classifiers trained with users' reported data. %As claimed in \cite{wang2019privacy}, aggregated noisy data also implies accumulation of the randomization in each reported data sample.As we will further perturb the utilized ADMM algorithm, aggregated noisy data plus further message perturbation makes information about the true labels that other servers have access to more obfuscated than those obtained by server $s_i$.
\begin{remark}
The $(\epsilon_p, \delta_p)$-DP guarantee is provided for users' feature vectors. Thus, in Phase~2, the sensitive information in those vectors is not disclosed much to the servers with lower trust degrees. For the labels, they obtain extra $(\epsilon_p, \delta_p)$-DP preservation in Phase~2. Since the privacy-preserving scheme in Phase~1 preserves $\epsilon$-DP for the labels, the released information about them in Phase~2 provides stronger privacy guarantee under the joint effect of $\epsilon$-DP in Phase~1 and $(\epsilon_p, \delta_p)$-DP in Phase~2. We will investigate the joint privacy-preserving degree in the future.
\end{remark}
\section{Performance Analysis} \label{performance_ana}
In this section, we analyze the performance of the classifiers trained by the proposed PDML framework. Note that three different uncertainties are introduced into the ADMM algorithm, and these uncertainties are coupled together. The difficulty in analyzing the performance lies in decomposing the effects of the three uncertainties and quantifying the role of each uncertainty. Further, it is also challenging to achieve perturbations mitigation on the trained classifiers, especially to mitigate the influence of users' wrong labels.

Here, we first give the definition of generalization error as the metric on the performance of the trained classifiers. Then, we establish a modified version of the loss function $\ell(\cdot)$, which simultaneously achieves uncertainty decomposition and  mitigation of label obfuscation. We finally derive a theoretical bound for the difference between the generalization error of trained classifiers and that of the ideal optimal classifier.
\subsection{Performance Metric}
To measure the quality of trained classifiers, we use generalization error for analysis, which describes the expected error of a classifier on future predictions \cite{shalev2008svm}. Recall that users' data samples are drawn from the unknown distribution $\mathcal{P}$. The generalization error of a classifier $\mathbf{w}$ is defined as the expectation of $\mathbf{w}$'s loss function with respect to $\mathcal{P}$ as
$\mathbb{E}_{(\mathbf{x},y)\sim\mathcal{P}} \left[\ell(y, \mathbf{w}^{\mathrm{T}}\mathbf{x})\right]$.
Further, define the regularized generalization error by
\begin{equation}\label{general_error}%\label{regu_general_error}
\small
J_{\mathcal{P}}(\mathbf{w}):=\mathbb{E}_{(\mathbf{x},y)\sim\mathcal{P}} \left[\ell(y, \mathbf{w}^{\mathrm{T}}\mathbf{x})\right]+\frac{a}{n} N(\mathbf{w}).
\end{equation}
We denote the classifier minimizing $J_{\mathcal{P}}(\mathbf{w})$ as $\mathbf{w}^\star$, i.e., $\mathbf{w}^\star:=\arg\min_{\mathbf{w}\in\mathcal{W}} J_{\mathcal{P}}(\mathbf{w})$. We call $\mathbf{w}^\star$ the ideal optimal classifier.

Here, $J_{\mathcal{P}}(\mathbf{w}^\star)$ is the reference regularized generalization error under the classifier class $\mathcal{W}$ and the used loss function $\ell(\cdot)$. The trained classifier can be viewed as a good predictor if it achieves generalization error close to $J_{\mathcal{P}}(\mathbf{w}^\star)$. Thus, as the performance metric of the classifiers, we use the difference between the generalization error of trained classifiers and $J_{\mathcal{P}}(\mathbf{w}^\star)$. The difference is denoted as $\Delta J_{\mathcal{P}}(\mathbf{w})$, that is, $\Delta J_{\mathcal{P}}(\mathbf{w}):=J_{\mathcal{P}}(\mathbf{w})-J_{\mathcal{P}}(\mathbf{w}^\star)$.% between $J_{\mathcal{P}}(\mathbf{w})$ and $J_{\mathcal{P}}(\mathbf{w}^\star)$ as

Furthermore, to measure the performance of the classifiers trained by different servers at multiple iterations, we introduce a comprehensive metric. First, considering that the classifiers $\mathbf{w}_i$ solved by server $s_i$ at different iterations may be different until the consensus constraint (\ref{constraint}) is satisfied, we define a classifier $\overline{\mathbf{w}}_i(t)$ to aggregate $\mathbf{w}_i$ in the first $t$ rounds as $\overline{\mathbf{w}}_i(t):=\frac{1}{t} \sum_{k=1}^{t} \mathbf{w}_i(k)$, where $\mathbf{w}_i(k)$ is the obtained classifier by solving (\ref{wi_update}). Moreover, due to the diversity of users' reported data, the classifiers solved by different servers may also differ (especially in the initial iterations). For this reason, we will later study the accumulated difference among the $n$ servers, that is, $\sum_{i=1}^{n} \Delta J_{\mathcal{P}}(\overline{\mathbf{w}}_i(t))$.
\subsection{Modified Loss Function in ADMM Algorithm} \label{sub_b}
To mitigate the effect of label obfuscation executed in Phase~1, we make some modification to the loss function $\ell(\cdot)$ in Problem~\ref{problem_1}. We use the noisy labels and the corresponding PPD $\epsilon$ in Phase~1 to adjust the loss function $\ell(\cdot)$ in (\ref{minimization_pro}). (Note that other parts of Problem~\ref{problem_1} are not affected by the noisy labels.) Define the modified loss function $\hat{\ell}(y'_{i, j}, \mathbf{w}_i^{\mathrm{T}}\mathbf{x}_{i, j}, \epsilon)$ by
\vspace{-0.3cm}
\begin{equation}\label{modified_loss}
\small
\hat{\ell}(y'_{i, j}, \mathbf{w}_i^{\mathrm{T}}\mathbf{x}_{i, j}, \epsilon):=\frac{e^\epsilon\ell(y'_{i, j}, \mathbf{w}_i^{\mathrm{T}}\mathbf{x}_{i, j})-\ell(-y'_{i, j}, \mathbf{w}_i^{\mathrm{T}}\mathbf{x}_{i, j})}{e^\epsilon-1}.
\end{equation}
This function has the following properties.
\begin{proposition} \label{unbiased_loss}
\begin{enumerate}[itemindent=0.3em, label=(\roman*),labelsep=0.3em]%[listparindent=-4em,itemindent=0em]
\item $\hat{\ell}(y'_{i, j}, \mathbf{w}_i^{\mathrm{T}}\mathbf{x}_{i, j}, \epsilon)$ is an unbiased estimate of $\ell(y_{i, j}, \mathbf{w}_i^{\mathrm{T}}\mathbf{x}_{i, j})$ as
\begin{equation}\label{unbiased_estimator}
\small
\mathbb{E}_{y'_{i, j}}\left[\hat{\ell}(y'_{i, j}, \mathbf{w}_i^{\mathrm{T}}\mathbf{x}_{i, j}, \epsilon)\right]=\ell(y_{i, j}, \mathbf{w}_i^{\mathrm{T}}\mathbf{x}_{i, j}).
\end{equation}
\item $\hat{\ell}(y'_{i, j}, \mathbf{w}_i^{\mathrm{T}}\mathbf{x}_{i, j}, \epsilon)$ is Lipschitz continuous with Lipschitz constant
\begin{equation}\label{lipschitz}
\small
\hat{c}_2:= \frac{e^{\epsilon}+1}{e^{\epsilon}-1}c_2,
\end{equation}
where $c_2$ is the bound of $\left|\frac{\partial\ell(\cdot)}{\partial \mathbf{w}_i}\right|$ given in Assumption \ref{loss_assumption}.
\end{enumerate}
\end{proposition}
The proof can be found in Appendix~\ref{proof_l1}.

Now, we make server $s_i$ use $\hat{\ell}(y'_{i, j}, \mathbf{w}_i^{\mathrm{T}}\mathbf{x}_{i, j}, \epsilon)$ in (\ref{modified_loss}) as the loss function. Thus, the objective function in (\ref{local_objective}) must be replaced with the one as follows:
\begin{equation}\label{modified_objective}
\small
\widehat{J}_i(\mathbf{w}_i):=\sum_{j=1}^{m_i}\frac{1}{m_i}\hat{\ell}(y'_{i, j}, \mathbf{w}_i^{\mathrm{T}}\mathbf{x}_{i, j}, \epsilon)+\frac{a}{n}N(\mathbf{w}_i).
\end{equation}
Similar to $J(\{\mathbf{w}_i\}_{i\in\mathcal{S}})$ in (\ref{original_objective}), we denote the objective function with the modified loss function as $\widehat{J}(\{\mathbf{w}_i\}_{i\in\mathcal{S}}):=\sum_{i=n}^n \widehat{J}_i(\mathbf{w}_i)$. Then, the following lemma holds, whose proof can be found in Appendix \ref{proof_l2}.
\begin{lemma}\label{objective_convex}
If the loss function $\ell(\cdot)$ and the regularizer $N(\cdot)$ satisfy Assumptions \ref{loss_assumption} and \ref{regularizer_assumption}, respectively, then $\widehat{J}(\{\mathbf{w}_i\}_{i\in\mathcal{S}})$ is $a\kappa$-strongly convex.
\end{lemma}

To simplify the notation, let $\hat{\kappa}:=a\kappa$. With the objective function $\widehat{J}(\{\mathbf{w}_i\}_{i\in\mathcal{S}})$, the whole optimization problem for finding a common classifier can be stated as follows:
\begin{problem} \label{problem_2}
\begin{small}
\begin{alignat}{2}
    \min_{\{\mathbf{w}_i\}} &  \quad \widehat{J}(\{\mathbf{w}_i\}_{i\in\mathcal{S}}) \nonumber \\
    \mathrm{s.t.}&  \quad \mathbf{w}_i=\mathbf{w}_l, \forall i, l. \nonumber
\end{alignat}
\end{small}
\end{problem}

\begin{lemma} \label{modified_solution}
Problem~\ref{problem_2} has an optimal solution set $\{\widehat{\mathbf{w}}_i\}_{i\in\mathcal{S}}\subset\mathcal{W}$ such that $\widehat{\mathbf{w}}_\mathrm{opt} = \widehat{\mathbf{w}}_i = \widehat{\mathbf{w}}_l, \forall i, l$.
\end{lemma}
Lemma~\ref{modified_solution} can be proved directly from Lemma~1 in \cite{forero2010consensus}, whose condition is satisfied by Lemma~\ref{objective_convex}.

We finally arrive at stating the optimization problem to be solved in this paper. To this end, for the modified objective function in (\ref{modified_objective}), we define the perturbed version as in (\ref{per_obj}) by $\widetilde{J}_i(\mathbf{w}_i):=\widehat{J}_i(\mathbf{w}_i)+\frac{1}{n}\boldsymbol{\eta}_i^{\mathrm{T}}\mathbf{w}_i$. Then, the whole objective function becomes
\begin{equation}\nonumber
\small
\widetilde{J}(\{\mathbf{w}_i\}_{i\in\mathcal{S}})=\sum_{i=1}^{n} \left[ \widehat{J}_i(\mathbf{w}_i)+\frac{1}{n}\boldsymbol{\eta}_i^{\mathrm{T}}\mathbf{w}_i\right].
\end{equation}
The problem for finding the classifier with randomized labels and perturbed objective functions is as follows:
\begin{problem} \label{problem_3}
\begin{small}
\begin{alignat}{2}
    \min_{\{\mathbf{w}_i\}} &  \quad \widetilde{J}(\{\mathbf{w}_i\}_{i\in\mathcal{S}}) \nonumber \\
    \mathrm{s.t.}&  \quad \mathbf{w}_i=\mathbf{w}_l, \forall i, l. \nonumber
\end{alignat}
\end{small}
\end{problem}
For $\widetilde{J}(\{\mathbf{w}_i\}_{i\in\mathcal{S}})$, we have the following lemma showing its convexity properties.
\begin{lemma}\label{J_tildle_convex}
$\widetilde{J}(\{\mathbf{w}_i\}_{i\in\mathcal{S}})$ is $\hat{\kappa}$-strongly convex. If $N(\cdot)$ satisfies that $\|\nabla^2 N(\cdot)\|_2\leq \varrho$, then $\widetilde{J}(\{\mathbf{w}_i\}_{i\in\mathcal{S}})$ has a $(nc_3+a\varrho)$-Lipschitz continuous gradient, where $c_3$ is the bound of $\frac{\partial\ell^2(\cdot)}{\partial \mathbf{w}^2}$ given in Assumption~\ref{loss_assumption}.
\end{lemma}
The proof can be found in Appendix~\ref{proof_l4}. For simplicity, we denote the Lipschitz continuous gradient of $\widetilde{J}(\mathbf{w})$ as $\varrho_{\widetilde{J}}$, namely, $\varrho_{\widetilde{J}}:=nc_3 + a\varrho$.

We now observe that Problem~\ref{problem_3} associated with the objective function $\widetilde{J}(\{\mathbf{w}_i\}_{i\in\mathcal{S}})$ has an optimal solution set $\{\widetilde{\mathbf{w}}_i\}_{i\in\mathcal{S}}\subset \mathcal{W}$ where
\begin{equation}\label{optimal_perturbation}
\small
\widetilde{\mathbf{w}}_\mathrm{opt} = \widetilde{\mathbf{w}}_i = \widetilde{\mathbf{w}}_l,  \forall i, l.
\end{equation}
In fact, this can be shown by an argument similar to Lemma~\ref{modified_solution}, where Lemma~\ref{J_tildle_convex} establishes the convexity of the objective function (as in Lemma~\ref{objective_convex}). %That is, Problem \ref{problem_3} has an optimal solution set such that
%By the two-phase PDML framework and the modified loss functionThen, the optimization problem under  is as follows:
\subsection{Generalization Error Analysis}
In this subsection, we analyze the the accumulated difference between the generalization error of trained classifiers and $J_{\mathcal{P}}(\mathbf{w}^\star)$, i.e., $\sum_{i=1}^{n} \Delta J_{\mathcal{P}}(\overline{\mathbf{w}}_i(t))$. For the analysis, we use the technique from \cite{li2017robust}, which considers the problem of ADMM learning in the presence of erroneous updates. Here, our problem is more complicated because besides the erroneous updates brought by primal variable perturbation, there is also uncertainty in the training data and the objective functions. All these uncertainties are coupled together, which brings extra challenges for performance analysis.

We first decompose $\Delta J_{\mathcal{P}}(\overline{\mathbf{w}}_i(t))$ in terms of different uncertainties. To do so, we must introduce a new regularized generalization error associated with the modified loss function $\hat{\ell}(y', \mathbf{w}^{\mathrm{T}}\mathbf{x}, \epsilon)$ and the noisy data distribution $\mathcal{P}_{\epsilon}$. Similar to (\ref{general_error}), for a classifier $\mathbf{w}$, it is defined by
\begin{equation}\nonumber
\small
J_{\mathcal{P}_{\epsilon}}(\mathbf{w})
 =\mathbb{E}_{(\mathbf{x},y')\sim\mathcal{P}_{\epsilon}} \left[\hat{\ell}(y', \mathbf{w}^{\mathrm{T}}\mathbf{x}, \epsilon)\right]+\frac{a}{n} N(\mathbf{w}).
\end{equation}
According to Proposition~\ref{unbiased_loss}, $\hat{\ell}(y', \mathbf{w}^{\mathrm{T}}\mathbf{x}, \epsilon)$ is an unbiased estimate of $\ell(y, \mathbf{w}^{\mathrm{T}}\mathbf{x})$. Thus, it is straightforward to obtain the following lemma, whose proof is omitted.
\begin{lemma}\label{equal_error}
For a classifier $\mathbf{w}$, we have $J_{\mathcal{P}_{\epsilon}}(\mathbf{w})=J_{\mathcal{P}}(\mathbf{w})$.
\end{lemma}
Now, we can decompose $\Delta J_{\mathcal{P}}(\overline{\mathbf{w}}_i(t))$ as follows:
\begin{equation}\label{deltaJ_wi}
\small
\begin{split}
& \Delta J_{\mathcal{P}}(\overline{\mathbf{w}}_i(t)) =J_{\mathcal{P}}(\overline{\mathbf{w}}_i(t))-J_{\mathcal{P}}(\mathbf{w}^\star) \\
& =J_{\mathcal{P}_{\epsilon}}(\overline{\mathbf{w}}_i(t))-J_{\mathcal{P}_{\epsilon}}(\mathbf{w}^\star) \\
& =\widetilde{J}_i({\overline{\mathbf{w}}_i}(t))-\widetilde{J}_i(\widetilde{\mathbf{w}}_\mathrm{opt}) + \widehat{J}_i(\widetilde{\mathbf{w}}_\mathrm{opt})-\widehat{J}_i(\widehat{\mathbf{w}}_\mathrm{opt}) \\
& \quad +\widehat{J}_i(\widehat{\mathbf{w}}_\mathrm{opt})-\widehat{J}_i({\mathbf{w}^\star}) +J_{\mathcal{P}_{\epsilon}}(\overline{\mathbf{w}}_i(t))-\widehat{J}_i(\overline{\mathbf{w}}_i(t)) \\
& \quad +\widehat{J}_i({\mathbf{w}^\star})-J_{\mathcal{P}_{\epsilon}}(\mathbf{w}^\star) + \boldsymbol{\eta}_i^{\mathrm{T}} (\widetilde{\mathbf{w}}_\mathrm{opt}-\overline{\mathbf{w}}_i(t)).
\end{split}
\end{equation}
We will analyze each term in the far right-hand side of (\ref{deltaJ_wi}). The term $\widetilde{\mathbf{w}}_\mathrm{opt}-\overline{\mathbf{w}}_i(t)$ describes the difference between the classifier $\overline{\mathbf{w}}_i(t)$ and the optimal solution $\widetilde{\mathbf{w}}_\mathrm{opt}$ to Problem~\ref{problem_3}. Before analyzing this difference, we first consider the deviation between the perturbed classifier $\widetilde{\mathbf{w}}_i(t)$ and $\widetilde{\mathbf{w}}_\mathrm{opt}$, and a bound for it can be obtained by \cite{li2017robust}.

Here, we introduce some notations related to the bound. Let the compact forms of vectors be $\widetilde{\mathbf{w}}(t):=[\widetilde{\mathbf{w}}_1^\mathrm{T}(t) \cdots \widetilde{\mathbf{w}}_n^\mathrm{T}(t)]^\mathrm{T}$, $\boldsymbol{\theta}(t):=[\boldsymbol{\theta}_1^\mathrm{T}(t)\cdots \boldsymbol{\theta}_n^\mathrm{T}(t)]^\mathrm{T}$, and $\boldsymbol{\eta} := [\boldsymbol{\eta}_1^\mathrm{T} \cdots \boldsymbol{\eta}_n^\mathrm{T}]^\mathrm{T}$. Also, let $\widehat{\mathbf{w}}^{*}:=[I_d \cdots I_d]^\mathrm{T}\cdot\widehat{\mathbf{w}}_\mathrm{opt}$, $\widetilde{\mathbf{w}}^{*}:=[I_d \cdots I_d]^\mathrm{T}\cdot\widetilde{\mathbf{w}}_\mathrm{opt}$, and $\overline{L}:=\frac{1}{2} (L_{+}+L_{-})$. An auxiliary sequence $\mathbf{r}(t)$ is defined as $\mathbf{r}(t) := \sum_{k=0}^{t} Q\widetilde{\mathbf{w}}(k)$ with $Q:=\bigl(\frac{L_{-}}{2}\bigr)^{\frac{1}{2}}$ \cite{makhdoumi2017convergence}. $\mathbf{r}(t)$ has an optimal value $\mathbf{r}_{\mathrm{opt}}$, which is the solution to the equation $Q\mathbf{r}_{\mathrm{opt}}+\frac{1}{2\beta} \nabla \widetilde{J}(\widetilde{\mathbf{w}}_\mathrm{opt})=0$.

%We introduce some notations to simplify the mathematical expression of the convergence results. Recall that $L_{+}$ and $L_{-}$ defined in Section \ref{admm_alg} are the matrices related to the underlying communication topology $\mathcal{G}$. Specifically, $L_{+}$ and $L_{-}$ are the extended signless and signed Laplacian matrices, respectively \cite{shi2014linear}. Both matrices are positive semi-definite (and thus $Q$ is also positive semi-definite). We then introduce the extended degree matrix of $\mathcal{G}$ as $W=\frac{1}{2} (L_{+}+L_{-})$. For a semi-definite matrix, .

Further, we define some important parameters to be used in the next lemma. The first two parameters, $b\in(0,1)$ and $\lambda_1>1$, are related to the underlying network topology $\mathcal{G}$ and will be used to establish convergence property of the perturbed ADMM algorithm. Let $\varphi := \frac{\lambda_1-1}{\lambda_1}\frac{2\hat{\kappa} \sigma_{\min}^2(Q)\sigma_{\min}^2(L_{+})}{\varrho_{\widetilde{J}}^2 \sigma_{\min}^2(L_{+})+ 2\hat{\kappa} \sigma_{\max}^2(L_{+})}$, where $\sigma_{\max}(\cdot)$ and $\sigma_{\min}(\cdot)$ denote the maximum and minimum nonzero eigenvalues of a matrix, respectively. Also, we define $M_1$ and $M_2$ with constant $\lambda_2>1$ as
\begin{equation}\nonumber
\small
\begin{split}
M_1 & := \frac{b(1+\varphi) \sigma_{\min}^2(L_{+}) (1-1/{\lambda_2})}{4b\sigma_{\min}^2(L_{+}) (1-1/{\lambda_2}) + 16\sigma_{\max}^2(\overline{L})}, \\
M_2 & := \frac{(1-b) (1+\varphi)\sigma_{\min}^2(L_{+}) - \sigma_{\max}^2(L_{+})}{4\sigma_{\max}^2(L_{+})+4(1-b)\sigma_{\min}^2(L_{+})}.
\end{split}
\end{equation}

Then, we have the following lemma from \cite{li2017robust}, which gives a bound for $\widetilde{\mathbf{w}}(t)-\widetilde{\mathbf{w}}^{*}$.
\begin{lemma}\label{classifier_converge}
Suppose that the conditions of Lemma \ref{J_tildle_convex} hold. If the parameters $b$ and $\lambda_1$ can be chosen such that
\begin{equation}\label{b_delta}
\small
(1-b)(1+\varphi)\sigma_{\min}^2(L_{+})-\sigma_{\max}^2(L_{+})>0.
\end{equation}
Take $\beta$ in (\ref{gammai_update}) as $\beta = \sqrt{\frac{\lambda_1 \lambda_3 (\lambda_4-1)\varrho_{\widetilde{J}}^2}{\lambda_4(\lambda_1-1)\sigma_{\max}^2(L_{+}) \sigma_{\min}^2(Q)}}$, where $\lambda_4:=1+\sqrt{\frac{\varrho_{\widetilde{J}}^2 \sigma_{\min}^2(L_{+}) +2\hat{\kappa} \sigma_{\max}^2(L_{+})}{\alpha \lambda_3 \varrho_{\widetilde{J}}^2 \sigma_{\min}^2(L_{+})}}$ with $0<\alpha< \min\{M_1, M_2\}$, and $\lambda_3 := 1+\frac{2\hat{\kappa} \sigma_{\max}^2(L_{+})}{\varrho_{\widetilde{J}}^2\sigma_{\min}^2(L_{+})}$. Then, it holds
\begin{equation}\label{classifer_bound}
\small
\left\|\widetilde{\mathbf{w}}(t)-\widetilde{\mathbf{w}}^{*}\right\|_2^2 \leq C^{t} \left(H_1 + \sum_{k=1}^{t} C^{-k} H_2 \|\boldsymbol{\theta}(k)\|_2^2\right),
\end{equation}
where $C := \frac{(1+4\alpha)\sigma_{\max}^2(L_{+})}{(1-b)(1+\varphi-4\alpha)\sigma_{\min}^2(L_{+})}$, and $H_1:= \left\|\mathbf{w}(0)-\widetilde{\mathbf{w}}^{*}\right\|_2^2+\frac{4}{(1+4\alpha)\sigma_{\max}^2(L_{+})}\left\|\mathbf{r}(0)- \mathbf{r}_{\mathrm{opt}}\right\|_2^2$, $H_2 := \frac{b(\lambda_2 -1)}{1-b} + \frac{\frac{4\varphi \lambda_1 \sigma_{\max}^2(\overline{L})}{\sigma_{\min}^2(Q)} + \sigma_{\max}^2(L_{+}) \left(\sqrt{\varphi} + \sqrt{\frac{2(\lambda_1-1)\sigma_{\min}^2(Q)}{\alpha \lambda_1 \lambda_3\varrho_{\widetilde{J}}^2}}\right)^2}{(1-b) (1+\varphi) (1+\varphi-4\alpha)\sigma_{\min}^2(L_{+})}$.
\end{lemma}
Lemma~\ref{classifier_converge} implies that given a connected graph $\mathcal{G}$ and the objective function in Problem~\ref{problem_3}, if the parameters $b$ and $\lambda_1$ satisfy (\ref{b_delta}), then $C$ in (\ref{classifer_bound}) is guaranteed to be less than 1. In this case, the obtained classifiers will converge to the neighborhood of the optimal solution $\widetilde{\mathbf{w}}_{\mathrm{opt}}$, where the radius of the neighborhood is $\lim_{t\rightarrow\infty} \sum_{k=1}^{t} C^{t-k} H_2 \|\boldsymbol{\theta}(k)\|_2^2$. The modified ADMM algorithm can achieve different radii depending on the added noises $\boldsymbol{\theta}(k)$. Since many parameters are involved, to meet the condition (\ref{b_delta}) may not be straightforward. In order to make $C$ smaller to achieve better convergence rate, in addition to the parameters, one may change, for example, the graph $\mathcal{G}$ to make the value $\frac{\sigma_{\max}^2(L_{+})}{\sigma_{\min}^2(L_{+})}$ smaller.

Theorem \ref{delta_Jt} to be stated below gives the upper bound of the accumulated difference $\sum_{i=1}^n \Delta J_{\mathcal{P}}(\overline{\mathbf{w}}_i(t))$ in the sense of expectation. In the theorem, we employ the important concept of Rademacher complexity \cite{shalev2014understanding}. It is defined on the classifier class $\mathcal{W}$ and the collected data used for training, that is, $\mathrm{Rad}_i(\mathcal{W}):=\frac{1}{m_i} \mathbb{E}_{\nu_j}\left[\sup_{\mathbf{w}\in \mathcal{W}} \sum_{j=1}^{m_i} \nu_j \mathbf{w}^\mathrm{T}\mathbf{x}_{i, j}\right]$, where $\nu_1, \nu_2, \ldots, \nu_{m_i}$ are independent random variables drawn from the Rademacher distribution, i.e., $\Pr (\nu_j=1)=\Pr (\nu_j=-1)=\frac{1}{2}$ for $j=1, 2, \ldots, m_i$. In addition, we use the notation $\|\mathbf{v}\|_{A}^2$ to denote the norm of a vector $\mathbf{v}$ with a positive definite matrix $A$, i.e., $\|\mathbf{v}\|_{A}^2=\mathbf{v}^{\mathrm{T}}A\mathbf{v}$.
\begin{theorem} \label{delta_Jt}
Suppose that the conditions in Lemma~\ref{classifier_converge} are satisfied and the decaying rate of noise variance is set as $\rho\in (0,C)$. Then, for $\epsilon>0$ and $\delta\in(0,1)$, the aggregated classifier $\overline{\mathbf{w}}_i(t)$ obtained by the privacy-aware ADMM scheme (\ref{wi_update})-(\ref{gammai_update}) satisfies with probability at least $1-\delta$
\begin{equation}\label{eq_delta_Jt}
\small
\begin{split}
& \mathbb{E}_{\{\boldsymbol{\theta}(k)\}} \left\{\sum_{i=1}^{n} \Delta J_{\mathcal{P}}(\overline{\mathbf{w}}_i(t))\right\} \leq \frac{1}{2t}\frac{C}{1-C}\left(H_1 + \frac{\rho H_2}{C- \rho}\sum_{i=1}^{n} d V_i^2\right) \\
& + \frac{n}{2}R^2 + \frac{\beta}{t}\left[H_3 + \left(\frac{\sigma_{\max}^2(L_{+})}{2\sigma_{\max}^2(L_{-})} + 2\sigma_{\max}^2(Q)\right)\frac{\sum_{i=1}^n d V_i^2}{1-\rho}\right]\\
&  +\frac{1}{n \hat{\kappa}}R^2 + 4 \frac{e^{\epsilon}+1}{e^{\epsilon}-1}\sum_{i=1}^n\left(c_2 \mathrm{Rad}_i(\mathcal{W})+2c_1\sqrt{\frac{ 2\ln(4/\delta)}{m_i}}\right),
\end{split}
\end{equation}
where $H_3 =\|\mathbf{r}(0)\|_2^2 + \|\mathbf{w}(0)-\widetilde{\mathbf{w}}^{*}\|_{\frac{L_+}{2}}^2$, and the parameters $C$, $H_1$, $H_2$ and $\beta$ are found in Lemma \ref{classifier_converge}.
\end{theorem}
\begin{proof}
In what follows, we evaluate the terms in the far right-hand side of (\ref{deltaJ_wi}) by dividing them into three groups. The first is the terms $J_{\mathcal{P}_{\epsilon}}(\overline{\mathbf{w}}_i(t))-\widehat{J}_i(\overline{\mathbf{w}}_i(t))+\widehat{J}_i({\mathbf{w}^\star})-J_{\mathcal{P}_{\epsilon}}(\mathbf{w}^\star)$. We can bound them from above as
\begin{equation}\nonumber
\small
\begin{split}
& J_{\mathcal{P}_{\epsilon}}(\overline{\mathbf{w}}_i(t))-\widehat{J}_i(\overline{\mathbf{w}}_i(t))+\widehat{J}_i({\mathbf{w}^\star})-J_{\mathcal{P}_{\epsilon}}(\mathbf{w}^\star) \\
& \leq2\max_{\mathbf{w}\in\mathcal{W}} \left|J_{\mathcal{P}_{\epsilon}}(\mathbf{w})-\widehat{J}_i({\mathbf{w}})\right|.
\end{split}
\end{equation}
According to Theorem 26.5 in \cite{shalev2014understanding}, with probability at least $1-\delta$, we have
\begin{equation}\nonumber
\small
\begin{split}
& \max_{\mathbf{w}\in\mathcal{W}} \left|J_{\mathcal{P}_{\epsilon}}(\mathbf{w})-\widehat{J}_i({\mathbf{w}})\right| \\
& \leq2\mathrm{Rad}_i(\hat{\ell}\circ\mathcal{W})+4\left|\hat{\ell}(y'_{i, j}, \mathbf{w}_i^{\mathrm{T}}\mathbf{x}_{i, j}, \epsilon)\right|\sqrt{\frac{2\ln (4/\delta)}{m_i}},
\end{split}
\end{equation}
where $\mathrm{Rad}_i(\hat{\ell}\circ\mathcal{W})$ is the Rademacher complexity of $\mathcal{W}$ with respect to $\hat{\ell}$. Further, by the contraction lemma in \cite{shalev2014understanding},
\begin{equation}\nonumber
\small
\mathrm{Rad}_i(\hat{\ell}\circ\mathcal{W})\leq \hat{c}_2\mathrm{Rad}_i(\mathcal{W})=c_2\frac{e^{\epsilon}+1}{e^{\epsilon}-1} \mathrm{Rad}_i(\mathcal{W}),
\end{equation}
where we have used Proposition~\ref{unbiased_loss}. Also, from (\ref{modified_loss}), we derive
\begin{equation}\nonumber
\small
\left|\hat{\ell}(y'_{i, j}, \mathbf{w}_i^{\mathrm{T}}\mathbf{x}_{i, j}, \epsilon)\right| \leq \frac{e^{\epsilon}+1}{e^{\epsilon}-1}c_1,
\end{equation}
where $c_1$ is the bound of the original loss function $\ell(\cdot)$ (Assumption \ref{loss_assumption}). Then, it follows that
\begin{equation}\label{rademacher_final}
\small
\begin{split}
& J_{\mathcal{P}_{\epsilon}}(\overline{\mathbf{w}}_i(t))-\widehat{J}_i({\overline{\mathbf{w}}_i}(t))+\widehat{J}_i({\mathbf{w}^\star})-J_{\mathcal{P}_{\epsilon}}(\mathbf{w}^\star) \\
& \leq 4\frac{e^{\epsilon}+1}{e^{\epsilon}-1}\left(c_2\mathrm{Rad}_i(\mathcal{W})+2c_1\sqrt{\frac{2\ln (4/\delta)}{m_i}}\right).
\end{split}
\end{equation}

The second group in (\ref{deltaJ_wi}) are the terms about $\widetilde{J}_i(\cdot)$ and $\widehat{J}_i(\cdot)$. In their aggregated forms, by Lemma~\ref{modified_solution}, it holds
\begin{equation}\label{summation_form}
\small
\begin{split}
&  \widetilde{J}(\overline{\mathbf{w}}(t))- \widetilde{J}(\widetilde{\mathbf{w}}^{*}) + \widehat{J}(\widetilde{\mathbf{w}}^{*})-\widehat{J}(\widehat{\mathbf{w}}^{*}) + \widehat{J}(\widehat{\mathbf{w}}^{*}) - \widehat{J}({\mathbf{w}^\star}) \\
& \leq \widetilde{J}(\overline{\mathbf{w}}(t))- \widetilde{J}(\widetilde{\mathbf{w}}^{*}) + \widehat{J}(\widetilde{\mathbf{w}}^{*})-\widehat{J}(\widehat{\mathbf{w}}^{*}) \\
& \leq \frac{1}{t} \sum_{k=1}^{t} \widetilde{J}(\mathbf{w}(k))-\widetilde{J}(\widetilde{\mathbf{w}}^{*}) + \widehat{J}(\widetilde{\mathbf{w}}^{*})-\widehat{J}(\widehat{\mathbf{w}}^{*}),
\end{split}
\end{equation}
where we have used Jensen's inequality given the strongly convex $\widetilde{J}(\cdot)$.
For the first two terms in (\ref{summation_form}), by Theorem~1 of \cite{li2017robust}, we have
\begin{equation}\label{taking_expectation}
\small
\begin{split}
& \frac{1}{t} \sum_{k=1}^{t} \widetilde{J}(\mathbf{w}(k))- \widetilde{J}(\widetilde{\mathbf{w}}^{*}) \leq \frac{\beta}{t} \|\mathbf{r}(0)\|_2^2 + \|\mathbf{w}(0)-\widetilde{\mathbf{w}}^{*}\|_{\frac{L_+}{2}}^2 \\
& + \frac{\beta}{t} \sum_{k=1}^{t} \left(\frac{\sigma_{\max}^2(L_{+})}{2\sigma_{\max}^2(L_{-})} \|\boldsymbol{\theta}(k)\|_2^2 + 2\boldsymbol{\theta}(k)^{\mathrm{T}}Q\mathbf{r}(k)\right).
\end{split}
\end{equation}
Take the expectation on both sides of (\ref{taking_expectation}) with respect to $\boldsymbol{\theta}(k)$. Given $\mathbb{E}_{\{\boldsymbol{\theta}(k)\}} \left\{\|\boldsymbol{\theta}(k)\|_2^2\right\}= \sum_{i=1}^{n} d V_i^2 \rho^{k-1}$, we derive
\begin{equation}\nonumber
\small
\begin{split}
\mathbb{E}_{\{\boldsymbol{\theta}(k)\}} \left\{2\boldsymbol{\theta}(k)^{\mathrm{T}}Q\mathbf{r}(k)\right\}& = \mathbb{E}_{\{\boldsymbol{\theta}(k)\}} \left\{2 \|Q \boldsymbol{\theta}(k)\|_2^2\right\} \\
& \leq 2\sigma_{\max}^2(Q) \sum_{i=1}^{n} d V_i^2 \rho^{k-1},
\end{split}
\end{equation}
where we used $\mathbb{E}\left\{\boldsymbol{\theta}(k)\right\}=0$ and $\mathbb{E}_{\{\boldsymbol{\theta}(k)\}}\left\{\boldsymbol{\theta}(k-1) \boldsymbol{\theta}(k)\right\}=0$. Thus, it follows that
\begin{equation}\nonumber
\small
\begin{split}
& \mathbb{E}_{\{\boldsymbol{\theta}(k)\}} \left\{ \sum_{k=1}^{t} \left(\frac{\sigma_{\max}^2(L_{+})}{2\sigma_{\max}^2(L_{-})} \|\boldsymbol{\theta}(k)\|_2^2 + 2\boldsymbol{\theta}(k)^{\mathrm{T}}Q\mathbf{r}(k)\right)\right\} \\
& \leq \left(\frac{\sigma_{\max}^2(L_{+})}{2\sigma_{\max}^2(L_{-})} + 2\sigma_{\max}^2(Q)\right)\sum_{i=1}^{n} d V_i^2 \sum_{k=1}^{t} \rho^{k-1} \\
& \leq \left(\frac{\sigma_{\max}^2(L_{+})}{2\sigma_{\max}^2(L_{-})} + 2\sigma_{\max}^2(Q)\right) \frac{\sum_{i=1}^{n} d V_i^2}{1-\rho}.
\end{split}
\end{equation}
Then, for (\ref{taking_expectation}), we arrive at
\begin{equation}\label{exp_J_tilde}
\small
\begin{split}
& \mathbb{E}_{\{\boldsymbol{\theta}(k)\}} \left\{  \widetilde{J}(\overline{\mathbf{w}}(t))- \widetilde{J}(\widetilde{\mathbf{w}}^{*})\right\} \\
& \leq  \frac{\beta}{t}\left[H_3 + \left(\frac{\sigma_{\max}^2(L_{+})}{2\sigma_{\max}^2(L_{-})} + 2\sigma_{\max}^2(Q)\right)\frac{\sum_{i=1}^n d V_i^2}{1-\rho}\right].
\end{split}
\end{equation}
Next, we focus on the latter two terms in (\ref{summation_form}). Due to (\ref{optimal_perturbation}), we have $\widetilde{J}(\widetilde{\mathbf{w}}^{*})\leq \widetilde{J}(\widehat{\mathbf{w}}^{*})$, which yields
\begin{equation}\nonumber
\small
\widehat{J}(\widetilde{\mathbf{w}}^{*})-\widehat{J}(\widehat{\mathbf{w}}^{*}) \leq \frac{1}{n} \boldsymbol{\eta}^{\mathrm{T}} (\widetilde{\mathbf{w}}^{*}-\widehat{\mathbf{w}}^{*}) \leq \frac{1}{n} \|\boldsymbol{\eta}\|\|\widetilde{\mathbf{w}}^{*}-\widehat{\mathbf{w}}^{*}\|,
\end{equation}
By Lemma 7 in \cite{chaudhuri2011differentially}, we obtain $\|\widetilde{\mathbf{w}}^{*}-\widehat{\mathbf{w}}^{*}\| \leq \frac{1}{n} \frac{\|\boldsymbol{\eta}\|}{\hat{\kappa}}$. It follows
\begin{equation}\label{bound_J_hat}
\small
\widehat{J}(\widetilde{\mathbf{w}}^{*})-\widehat{J}(\widehat{\mathbf{w}}^{*}) \leq \frac{1}{\hat{\kappa}} \frac{\|\boldsymbol{\eta}\|^2}{n^2} \leq \frac{1}{n \hat{\kappa}} R^2,
\end{equation}
where $R$ is the bound of noise $\boldsymbol{\eta}_i$. Substituting (\ref{exp_J_tilde}) and (\ref{bound_J_hat}) into (\ref{summation_form}), we derive
\begin{equation}\label{exp_objective}
\small
\begin{split}
& \mathbb{E}_{\{\boldsymbol{\theta}(k)\}} \{\widetilde{J}(\overline{\mathbf{w}}(t))- \widetilde{J}(\widetilde{\mathbf{w}}^{*}) + \widehat{J}(\widetilde{\mathbf{w}}^{*})-\widehat{J}(\widehat{\mathbf{w}}^{*}) + \widehat{J}(\widehat{\mathbf{w}}^{*}) - \widehat{J}({\mathbf{w}^\star}) \} \\
& \leq \frac{1}{n \hat{\kappa}} R^2 + \frac{\beta}{t}\left[H_3 + \left(\frac{\sigma_{\max}^2(L_{+})}{2\sigma_{\max}^2(L_{-})} + 2\sigma_{\max}^2(Q)\right)\frac{\sum_{i=1}^n d V_i^2}{1-\rho}\right].
\end{split}
\end{equation}

The third group in (\ref{deltaJ_wi}) is the term $\boldsymbol{\eta}^{\mathrm{T}}(\widetilde{\mathbf{w}}^{*}-\overline{\mathbf{w}}(t))$. We have
\begin{equation}\nonumber
\small
\boldsymbol{\eta}^{\mathrm{T}}(\widetilde{\mathbf{w}}^{*}-\overline{\mathbf{w}}(t)) = \boldsymbol{\eta}^{\mathrm{T}} \left(\widetilde{\mathbf{w}}^{*}-\frac{1}{t} \sum_{k=1}^{t}(\widetilde{\mathbf{w}}(k)-\boldsymbol{\theta}(k))\right).
\end{equation}
Taking the expectation with respect to $\boldsymbol{\theta}(k)$, we obtain
\begin{equation}\nonumber
\small
\begin{split}
& \mathbb{E}_{\{\boldsymbol{\theta}(k)\}} \left\{\boldsymbol{\eta}^{\mathrm{T}}(\widetilde{\mathbf{w}}^{*}-\overline{\mathbf{w}}(t))\right\} \\
& = \mathbb{E}_{\{\boldsymbol{\theta}(k)\}} \left\{\boldsymbol{\eta}^{\mathrm{T}}\left(\widetilde{\mathbf{w}}^{*}- \frac{1}{t} \sum_{k=1}^{t} \widetilde{\mathbf{w}}(k)\right)\right\} \\
& \leq \frac{1}{2} \|\boldsymbol{\eta}\|_2^2 + \mathbb{E}_{\{\boldsymbol{\theta}(k)\}} \left\{\frac{1}{2t^2} \left\|\sum_{k=1}^{t} (\widetilde{\mathbf{w}}^{*}-\widetilde{\mathbf{w}}(k)) \right\|_2^2\right\} \\
& \leq \frac{1}{2} \|\boldsymbol{\eta}\|_2^2 + \mathbb{E}_{\{\boldsymbol{\theta}(k)\}} \left\{\frac{1}{2t} \sum_{k=1}^{t} \left\|\widetilde{\mathbf{w}}^{*}-\widetilde{\mathbf{w}}(k)\right\|_2^2\right\}.
\end{split}
\end{equation}
By Lemma \ref{classifier_converge}, we have
\begin{equation}\nonumber
\small
\left\|\widetilde{\mathbf{w}}^{*}-\widetilde{\mathbf{w}}(k)\right\|_2^2 \leq C^{t} \left(H_1 + \sum_{k=1}^{t} C^{-k} H_2 \|\boldsymbol{\theta}(k)\|_2^2\right).
\end{equation}
Then, it follows that
\begin{equation}\label{exp_eta_w}
\small
\begin{split}
& \mathbb{E}_{\{\boldsymbol{\theta}(k)\}} \left\{\boldsymbol{\eta}^{\mathrm{T}}(\widetilde{\mathbf{w}}^{*}-\overline{\mathbf{w}}(t))\right\} \\
& \leq \frac{1}{2} \|\boldsymbol{\eta}\|_2^2 + \frac{1}{2t} \mathbb{E}_{\{\boldsymbol{\theta}(k)\}} \left\{\sum_{k=1}^{t} C^{t} \left(H_1 + \sum_{k=1}^{t} C^{-k} H_2 \|\boldsymbol{\theta}(k)\|_2^2\right) \right\} \\
& \leq \frac{n}{2} R^2 + \frac{1}{2t}\frac{C}{1-C} \left(H_1 + \frac{\rho H_2}{C- \rho}\sum_{i=1}^{n} d V_i^2\right),
\end{split}
\end{equation}
where we have used $0<\rho<C$. Substituting (\ref{rademacher_final}), (\ref{exp_objective}) and (\ref{exp_eta_w}) into (\ref{deltaJ_wi}), we arrive at the bound in (\ref{eq_delta_Jt}).
\end{proof}
Theorem \ref{delta_Jt} provides a guidance for both users and servers to obtain a classification model with desired performance. In particular, the effects of three uncertainties on the bound of $\sum_{i=1}^n \Delta J_{\mathcal{P}}(\overline{\mathbf{w}}_i(t))$ have been successfully decomposed. Note that these effects are not simply superimposed but coupled together. Specifically, the terms in (\ref{eq_delta_Jt}) related to the primal variable perturbation decrease with iterations at the rate of $O\left(\frac{1}{t}\right)$. This also implies that the whole framework achieves convergence in expectation at this rate.

Compared with \cite{ding2019optimal} and \cite{li2017robust}, where bounds of $\frac{1}{t} \sum_{k=1}^{t} \widetilde{J}(\mathbf{w}(k))- \widetilde{J}(\widetilde{\mathbf{w}}^{*})$ are provided, we derive the difference between the generalization error of the aggregated classifier $\overline{\mathbf{w}}(t)$ and that of the ideal optimal classifier $\mathbf{w}^{\star}$, which is moreover given in a closed form. The bound in (\ref{eq_delta_Jt}) contains the effect of the unknown data distribution $\mathcal{P}$ while the bound of $\frac{1}{t} \sum_{k=1}^{t} \widetilde{J}(\mathbf{w}(k))- \widetilde{J}(\widetilde{\mathbf{w}}^{*})$ covers only the role of existing data. Although \cite{zhang2017dynamic} also considers the generalization error of found classifiers, no closed form of the bound is given, and the obtained bound may not decrease with iterations since the reference classifier therein is not $\mathbf{w}^{\star}$ but a time-varying one. In the more centralized setting of \cite{chaudhuri2011differentially}, $\Delta J_{\mathcal{P}}(\mathbf{w})$ is analyzed for the derived classifier $\mathbf{w}$, but there is no convergence issue since $\mathbf{w}$ is perturbed and published only once. %However, in this paper, we provide a closed form of $\sum_{i=1}^{n} \Delta J_{\mathcal{P}}(\overline{\mathbf{w}}_i(t))$ under an ADMM-based setting, and the bound decays with iterations by the rate of $\mathcal{O}\left(\frac{1}{t}\right)$.

Moreover, different from the works \cite{chaudhuri2011differentially,zhang2017dynamic,ding2019optimal} and \cite{li2017robust}, our analysis considers the effects of the classifier class $\mathcal{W}$ by Rademacher complexity. Such effects have been used in \cite{shalev2014understanding} in non-private centralized machine learning scenarios. Furthermore, in the privacy-aware (centralized or distributed) frameworks of \cite{chaudhuri2011differentially,zhang2017dynamic,ding2019optimal} and the robust ADMM scheme for erroneous updates \cite{li2017robust}, there is only one type of noise perturbation, and the uncertainty in the training data is not considered. %Conversely, in this paper, we consider three types of noise perturbations including the local randomization on users' labels. The effects of these perturbations are not simply superimposed but coupled together. To solve this new problem, we further provide a decomposition approach.

%{\color{blue}{From  The terms in (\ref{rademacher_final}) (contained in (\ref{eq_delta_Jt})) reflect the influence of classifier class and data quality, where the effects of local randomization in Phase 1 is included. Further, $(\frac{n}{2}+\frac{1}{n \hat{\kappa}}) R^2$ is brought by the objective function perturbation in Phase 2. It indicates that the offset of the objective function plays an important role on the trained classifiers' performance. Other parts in (\ref{eq_delta_Jt}) are the deviation caused by primal variable perturbation in Phase 2.}} %If bounded noises are utilized for perturbation and the noise bounds decay with iterations, then the bound for  $\left\|\widetilde{\mathbf{w}}(t)-\widetilde{\mathbf{w}}^{*}\right\|_2^2$ in (\ref{eq_delta_Jt}) will be deterministic. That is, the classifiers solved by all servers exponentially converge to $\widetilde{\mathbf{w}}_{\mathrm{opt}}$. However, perturbation with decaying bounded noises also provides more information for privacy violators to infer $\nabla \widehat{J}_i(\cdot)$, which weakens the privacy preservation in Phase 2.
% {-0.45cm}
\subsection{Comparisons and Discussions} \label{com_dis}
Here, we compare the proposed framework with existing schemes from the perspective of privacy and performance, and discuss how each parameter contributes to the results.

First, we find that the bound in (\ref{eq_delta_Jt}) is larger than those in \cite{chaudhuri2011differentially,zhang2017dynamic,ding2019optimal} if we adopt the approach in this paper to conduct performance analysis on these works. This is obvious since there are more perturbations in our setting. However, as we have discussed in Section \ref{discussion_privacy}, these existing frameworks do not meet the heterogeneous privacy requirements, and some of them cannot avoid accumulation of privacy losses, resulting in no protection at all. It should be emphasized that extra performance costs must be paid when the data contributors want to obtain stronger privacy guarantee. These existing frameworks may be better than ours in the sense of performance, but the premise is that users accept the privacy preservation provided by them. If users require heterogenous privacy protection, our framework can be more suitable.

Further, compared with \cite{chaudhuri2011differentially,zhang2017dynamic,ding2019optimal}, \cite{li2017robust} and \cite{shalev2014understanding}, we provide a more systematic result on the performance analysis in Theorem~\ref{delta_Jt}, where most parameters related to useful measures of classifiers (also privacy preservation) are included. Servers and users can set these parameters as needed, and thus obtain classifiers which can appropriately balance the privacy and the performance. We will discuss the roles of these parameters after some further analysis on the theoretical result.

According to Lemma~\ref{classifier_converge}, the classifiers solved by different servers converge to $\widetilde{\mathbf{w}}_\mathrm{opt}$ in the sense of expectation. The performance of $\widetilde{\mathbf{w}}_\mathrm{opt}$ can be analyzed in a similar way as in Theorem~\ref{delta_Jt}. This is given in the following corollary.
\begin{corollary} \label{corollary1}
For $\epsilon>0$ and $\delta\in(0,1)$, with probability at least $1-\delta$, we have
\begin{equation}\label{Jp_w_tilde}
\small
\begin{split}
& \Delta J_{\mathcal{P}}(\widetilde{\mathbf{w}}_\mathrm{opt}) \\
& \leq \frac{4}{n}\frac{e^{\epsilon}+1}{e^{\epsilon}-1} \sum_{i=1}^n \left(c_2 \mathrm{Rad}_i(\mathcal{W})+2c_1\sqrt{\frac{ 2\ln(4/\delta)}{m_i}}\right)+ \frac{1}{n\hat{\kappa}}R^2.
\end{split}
\end{equation}
\end{corollary}

For the sake of comparison, the next theorem provides a performance analysis when the privacy-preserving approach in Phase~2 is removed, and a corresponding result on the bound of $\Delta J_{\mathcal{P}}(\widehat{\mathbf{w}}_\mathrm{opt})$ is given in the subsequent corollary.
\begin{theorem} \label{thm3}
For $\epsilon>0$ and $\delta\in(0,1)$, the aggregated classifier $\overline{\mathbf{w}}_i(t)$ obtained by the original ADMM scheme (\ref{new_primal_local}) and (\ref{new_dual_local}) satisfies with probability at least $1-\delta$
\begin{equation}\label{delta_Jt_unper}
\small
\begin{split}
\sum_{i=1}^{n} \Delta J_{\mathcal{P}}(\overline{\mathbf{w}}_i(t)) & \leq \frac{\beta}{t} \left(\|\mathbf{w}(0)-\widehat{\mathbf{w}}^{*}\|_{\frac{L_+}{2}}^2 + \|\mathbf{r}(0)\|_2^2\right) \\
& + 4 \frac{e^{\epsilon}+1}{e^{\epsilon}-1}\sum_{i=1}^n\left(c_2 \mathrm{Rad}_i(\mathcal{W})+2c_1\sqrt{\frac{ 2\ln(4/\delta)}{m_i}}\right).
\end{split}
\end{equation}
\end{theorem}
\begin{corollary} \label{corollary2}
For $\epsilon>0$ and $\delta\in(0,1)$, with probability at least $1-\delta$, we have
\begin{equation}\label{Jp_w_hat}
\small
\Delta J_{\mathcal{P}}(\widehat{\mathbf{w}}_\mathrm{opt})
\leq\frac{4}{n}\frac{e^{\epsilon}+1}{e^{\epsilon}-1} \sum_{i=1}^n \left(c_2 \mathrm{Rad}_i(\mathcal{W})+2c_1\sqrt{\frac{ 2\ln(4/\delta)}{m_i}}\right).
\end{equation}
\end{corollary}

It is observed that the bound in (\ref{delta_Jt_unper}) is not in expectation since there is no noise perturbation during the ADMM iterations. It is interesting to note that the convergence rate of the unperturbed ADMM algorithm is also $O(\frac{1}{t})$. This implies that the modified ADMM algorithm preserves the convergence speed of the general distributed ADMM scheme.

However, there exists a tradeoff between performance and privacy protection. Comparing (\ref{eq_delta_Jt}) and (\ref{delta_Jt_unper}), we find that the extra terms in (\ref{eq_delta_Jt}) are the results of perturbations in Phase~2. Also, the effect of the objective function perturbation is reflected in (\ref{Jp_w_tilde}), that is, the term $\frac{1}{n\hat{\kappa}}R^2$. When $R$ (the bound of $\boldsymbol{\eta}_i$) increases, the generalization error of the trained classifier would increase as well, indicating worse performance. Similarly, if we use noises with larger initial variances and decaying rates to perturb the solved classifiers in each iteration, the bound in (\ref{eq_delta_Jt}) will also increase. %Despite the performance degradation, these perturbations and noises bring more obstructions for privacy violators to make inference.

\textbf{Effect of data quality}. We observe that the bound of $\Delta J_{\mathcal{P}}(\widehat{\mathbf{w}}_\mathrm{opt})$ in (\ref{Jp_w_hat}) also appears in (\ref{eq_delta_Jt}), (\ref{Jp_w_tilde}) and (\ref{delta_Jt_unper}). This bound reflects the effect of users' reported data, whose labels are randomized in Phase~1. It can be seen that besides the probability $\delta$, the bound in (\ref{Jp_w_hat}) is affected by three factors: PPD $\epsilon$, Rademacher complexity $\mathrm{Rad}_i(\mathcal{W})$, and the number of data samples $m_i$. Here, we discuss the roles of these factors. % performance of trained

For the effect of PPD, we find that when $\epsilon$ is small, the bound will decrease with an increase in $\epsilon$. However, when $\epsilon$ is sufficiently large, it has limited influence on the bound. In particular, by taking $\epsilon\rightarrow\infty$, the bound reduces to that for the optimal solution of Problem~\ref{problem_1}, where $(e^{\epsilon}+1)/(e^{\epsilon}-1)$ goes to 1 in (\ref{Jp_w_hat}). Note that $\mathrm{Rad}_i(\mathcal{W})$ and $m_i$ still remain and affect the performance. %(without privacy protection) as $\mathbf{w}_\mathrm{opt}$. We can obtain a bound for $\Delta J_{\mathcal{P}}(\mathbf{w}_\mathrm{opt})$ given by $\frac{4}{n} \sum_{i=1}^n \left(c_2 \mathrm{Rad}_i(\mathcal{W})+2c_1\sqrt{\frac{ 2\ln(4/\delta)}{m_i}}\right)$. Compared with those in (\ref{Jp_w_tilde}) and , the bound of $\Delta J_{\mathcal{P}}(\mathbf{w}_\mathrm{opt})$ is not affected by the PPD $\epsilon$. Meanwhile, it is noticed that are two underlying factors affecting the performance of the trained classifiers, even without privacy protection. These two are discussed further below. %Lemma \ref{modified_solution} also gives the guarantee that in this case, the classifiers obtained by different servers deterministically converge to $\mathbf{w}_\mathrm{opt}$.
%\begin{equation}\label{bound_nopri}
%\small
%\Delta J_{\mathcal{P}}(\mathbf{w}_\mathrm{opt}) \leq.
%\end{equation}

For the effect of $\mathrm{Rad}_i(\mathcal{W})$, we observe that the generalization errors of trained classifiers may become larger when $\mathrm{Rad}_i(\mathcal{W})$ increases. The Rademacher complexity is directly related to the size of the classifier class $\mathcal{W}$. If there are only a small number of candidate classifiers in $\mathcal{W}$, the solutions have a high probability of obtaining smaller deviation between their generalization errors and the reference generalization error $J_{\mathcal{P}}(\mathbf{w}^{\star})$. Nevertheless, we should guarantee the richness of the class $\mathcal{W}$ to make $J_{\mathcal{P}}(\mathbf{w}^{\star})$ small since $\mathbf{w}^\star$ trained in terms of $\mathcal{W}$ will have large generalization error. Though the deviation $\Delta J_{\mathcal{P}}(\cdot)$ may be small, the trained classifiers are not good predictors due to the bad performance of $\mathbf{w}^\star$.  Thus, setting an appropriate classifier class is important for obtaining a classifier with qualified performance. %when the size of $\mathcal{W}$ is extremely small.Such issue is out of the scope of this paper.

Finally, we consider the effect of the number of users. From the bound of $\Delta J_{\mathcal{P}}(\widehat{\mathbf{w}}_\mathrm{opt})$ in (\ref{Jp_w_hat}), we know that if $m_i$ becomes larger, the last term of the bound will decrease. In general, more data samples imply access to more information about the underlying distribution $\mathcal{P}$. Then, the trained classifier can predict the labels of newly sampled data from $\mathcal{P}$ with higher accuracy. Moreover, it can be seen that the bound is the average of $n$ local errors generated in different servers. When new servers participate in the DML framework, these servers should make sure that they have collected sufficient amount of training data samples. Otherwise, the bound may not decrease though the total number of data samples increases. This is because unbalanced local errors may lead to an increase in their average, implying larger bound of $\Delta J_{\mathcal{P}}(\cdot)$.
\begin{figure}
\vspace*{-2mm}
\begin{center}
\subfigure[Norm of the errors]{\label{consensus_obs}
\includegraphics[scale=0.45]{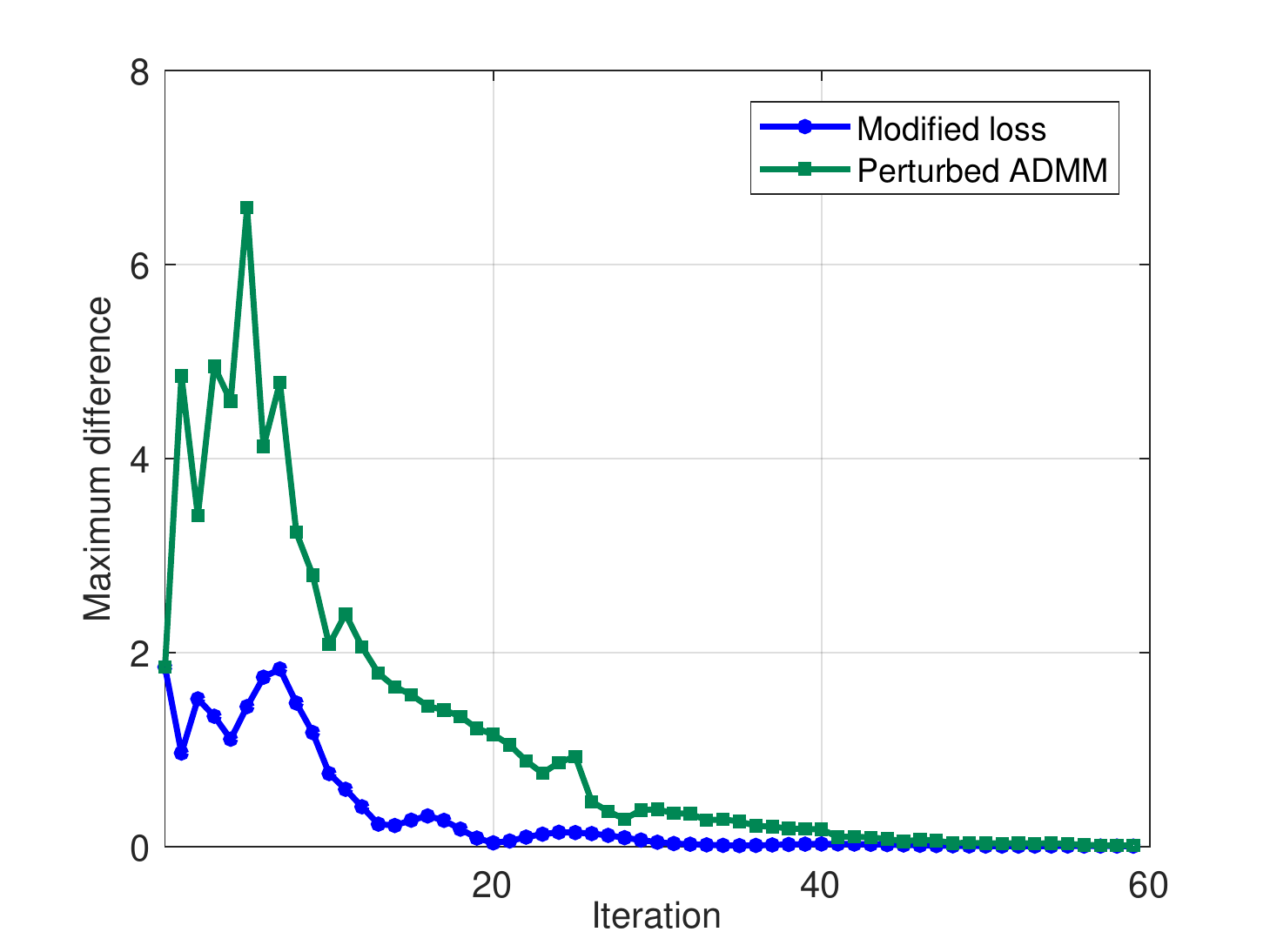}}
\subfigure[Empirical risks]{\label{com_risks}
\includegraphics[scale=0.45]{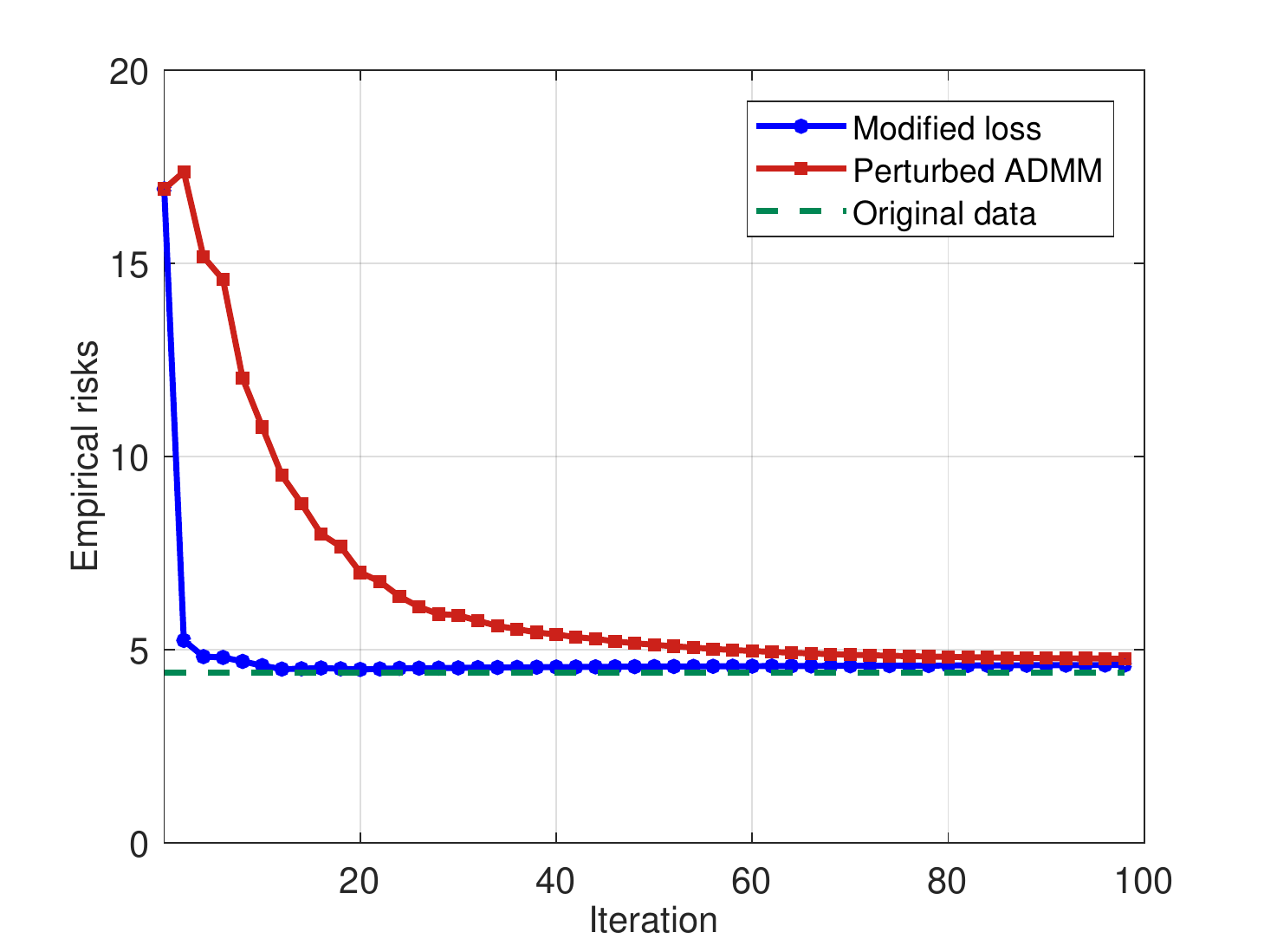}}
\caption{{\small The convergence properties of PDML.}}
\end{center}
\end{figure}
% {-0.2cm}
\section{Experimental Evaluation} \label{evaluation}
In this section, we conduct experiments to validate the obtained theoretical results and study the classification performance of the proposed PDML framework. Specifically, we first use a real-world dataset to verify the convergence property of the PDML framework and study how key parameters would affect the performance. Also, we leverage another seven datasets to verify the classification accuracy of the classifiers trained by the framework.
\subsection{Experiment Setup}
\subsubsection{Datasets} We use two kinds of publicly available datasets as described below to validate the convergence property and classification accuracy of the PDML.

(i) Adult dataset \cite{Dua2019}. The dataset contains census data of 48,842 individuals, where there are 14 attributes (e.g., age, work-class, education, occupation and native-country) and a label indicating whether a person's annual income is over \$50,000. After removing the instances with missing values, we obtain a training dataset with 45,222 samples. To preprocess the dataset, we adopt unary encoding approach to transform the categorial attributes into binary vectors, and further normalize the whole feature vector to be a vector with maximum norm of 1. The preprocessed feature vector is a 105-dimensional vector. For the labels, we mark the annual income over \$50,000 as 1, otherwise it is labeled as $-1$.

(ii) Gunnar R\"{a}tsch's benchmark datasets \cite{ucidata}. There are thirteen data subsets from the UCI repository in the benchmark datasets. To mitigate the effect of data quality, we select seven datasets with the largest data sizes to conduct experiments. The seven datasets are \textit{German}, \textit{Image}, \textit{Ringnorm}, \textit{Banana}, \textit{Splice}, \textit{Twonorm} and \textit{Waveform}, where the numbers of instances are 1,000, 2,086, 7,400, 5,300, 2,991, 7,400 and 5,000, respectively. Each dataset is partitioned into training and test data, with a ratio of approximately $70\%:30\%$.
\subsubsection{Underlying classification approach} Logistic regression (LR) is utilized for training the prediction model, where the loss function and regularizer are $\ell_{LR} (y_{i, j}, \mathbf{w}_i^\mathrm{T}\mathbf{x}_{i, j})= \log \bigl(1+e^{-y_{i, j} \mathbf{w}_i^\mathrm{T}\mathbf{x}_{i, j}}\bigr)$ and $N(\mathbf{w}_i) = \frac{1}{2}\|\mathbf{w}_i\|^2$, respectively. Then, the local objective function is given by
\begin{equation}\nonumber
\small
J_i(\mathbf{w}_i) = \sum_{j=1}^{m_i}\frac{1}{m_i}\log \left(1+e^{-y_{i, j} \mathbf{w}_i^\mathrm{T}\mathbf{x}_{i, j}}\right)+\frac{a}{2n}\|\mathbf{w}_i\|^2.
\end{equation}
It is easy to check that when the classifier class $\mathcal{W}$ is bounded (e.g., a bounded set $\mathcal{W}= \{\mathbf{w}\in\mathbb{R}^d\; |\; \|\mathbf{w}\|\leq W\}$), $\ell_{LR} (\cdot)$ satisfies Assumption~\ref{loss_assumption}. Due to the convexity property of $N(\mathbf{w}_i)$, $J_i(\mathbf{w}_i)$ is strongly convex. Then, according to Lemma~\ref{modified_solution}, Problems~\ref{problem_2} and \ref{problem_3} have optimal solution sets, and thus, we can use LR to train the classifiers.
\subsubsection{Network topology} We consider $n=10$ servers collaboratively train a prediction model. A connected random graph is used to describe the communication topology of the 10 servers. The used graph has $E=13$ communication links in total. Each server is responsible for collecting the data from a group of users, and thus there are 10 groups of users. In the experiments, we assume that each group has the same number of users, that is, $m_i=m_l, \forall i, l$. For example, we use $m=45,000$ instances sampled from the Adult dataset to train the classifier, and then each server collects data from $m_i = 4,500$ users.

\subsection{Experimental Results with Adult Dataset}
Based on the Adult dataset, we first verify the convergence property of the PDML framework. Fig.~\ref{consensus_obs} illustrates the maximum distances between the norms of arbitrary two classifiers found by different servers. We set the bound of $\boldsymbol{\eta}_i$ to 1. Other settings are the same as those with experiments under the synthetic dataset. For the sake of comparison, we also draw the variation curve (with circle markers) of the maximum distance when the privacy-preserving approach in Phase~2 is removed. We observe that both distances converge to 0, implying that the consensus constraint is eventually satisfied.

Fig. \ref{com_risks} shows the variation of empirical risks (the objective function in (\ref{original_objective})) as iterations proceed. Here, the green dashed line depicts the final empirical risk achieved by general ADMM with original data, which we call the reference empirical risk. There are also two curves showing varying empirical risks with privacy preservation. Comparing the two curves, we find that the ADMM with combined noise-adding scheme preserves the convergence property of the general ADMM algorithm. Due to the noise perturbations in Phase~2, the convergence time becomes longer. In addition, it can be seen that regardless of whether the privacy-preserving approach in Phase~2 is used, both ADMM schemes cannot achieve the same final empirical risks with that of the green line, which is consistent with the analysis in Section~\ref{com_dis}.
\begin{figure}
%\vspace*{-5mm}
\begin{center}
\subfigure[Different $R$ (with $\rho=0.8$)]{\label{effects_M}
\includegraphics[scale=0.45]{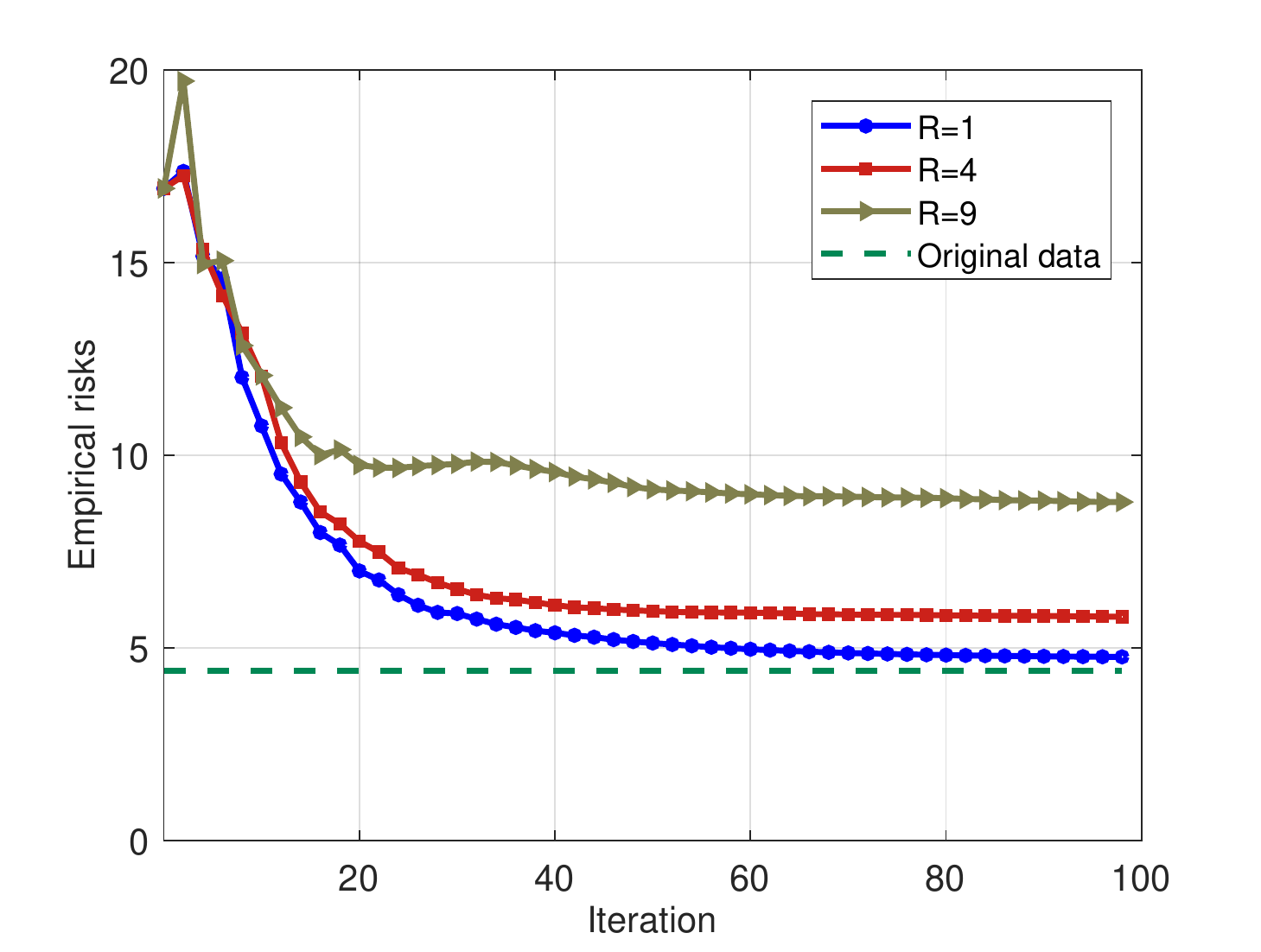}}
\subfigure[Different $\rho$ (with $R=1$)]{\label{effects_rho}
\includegraphics[scale=0.45]{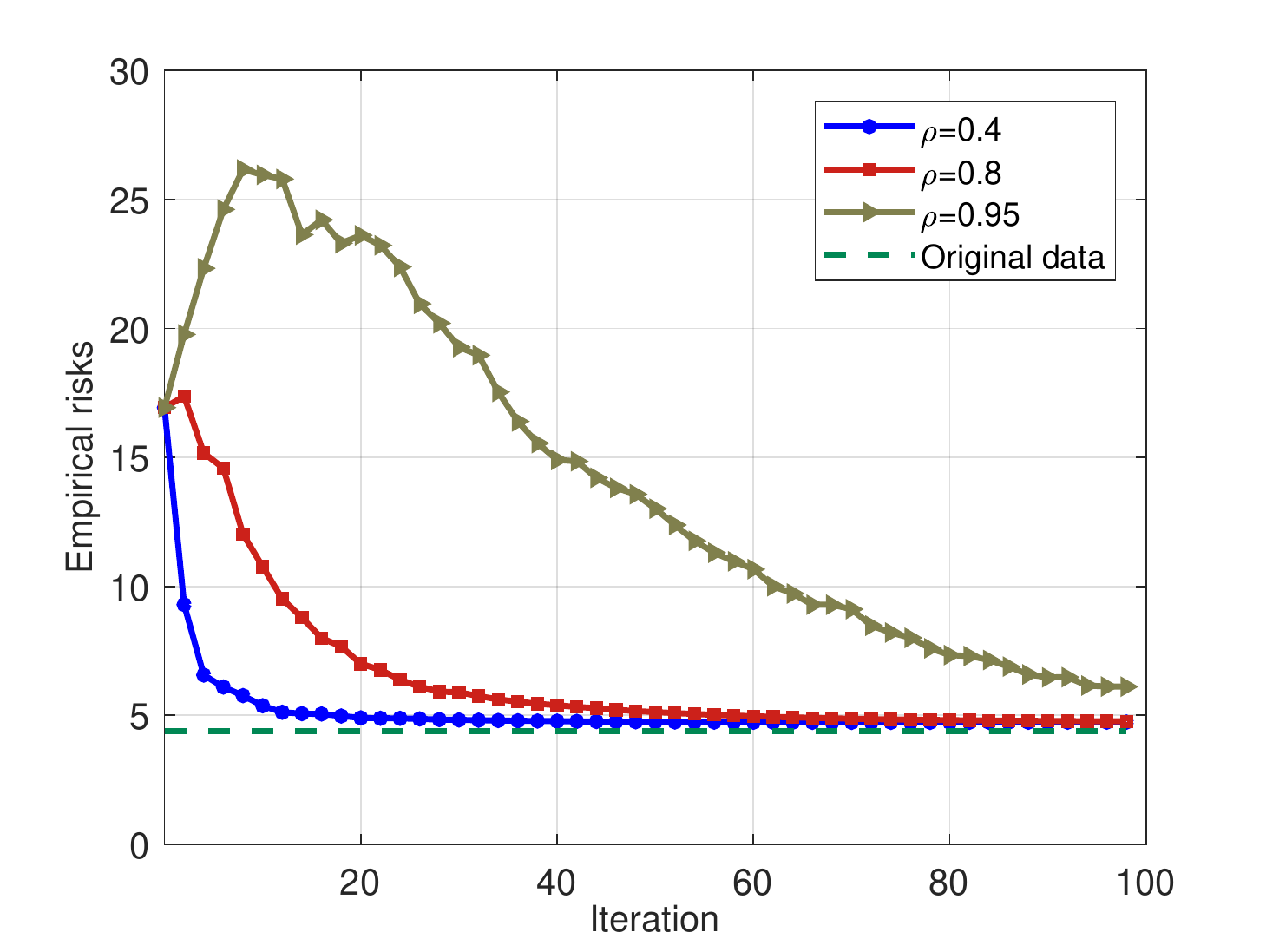}}
\subfigure[Different $\epsilon$]{\label{effects_epsilon}
\includegraphics[scale=0.45]{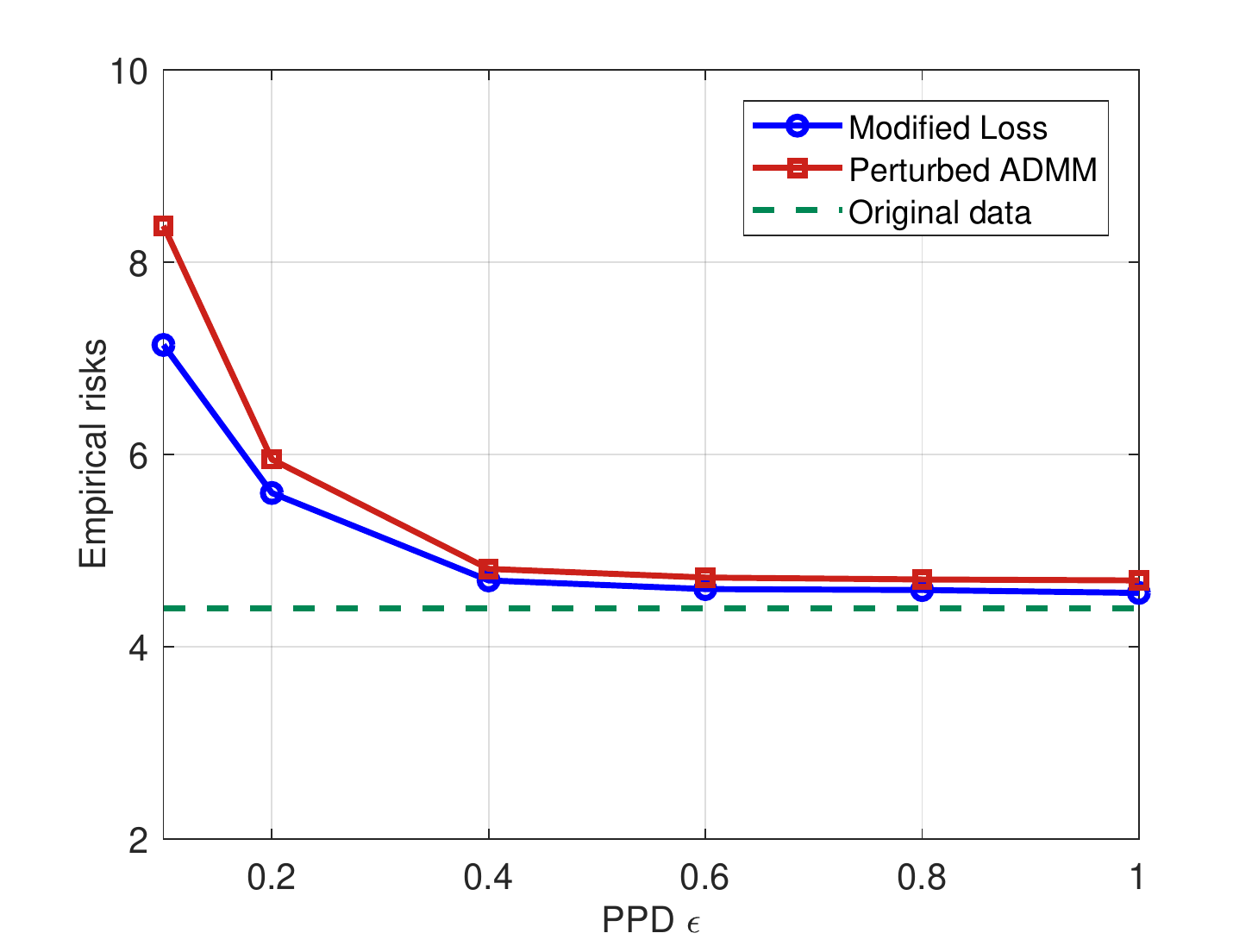}}
\caption{{\small The effects of key parameters.}}
\end{center}
\end{figure}
%\begin{figure}
%\begin{center}
%\includegraphics[scale=0.45]{figure/convergence_com.pdf}
%\caption{Comparison of empirical risks.} \label{com_risks}
%\end{center}
%\end{figure}
\begin{table*}[thb!]
%\setlength{\abovecaptionskip}{15pt}
%\vspace{0.8cm}
\caption{{\small Classification accuracy with test data (\%)}} \label{table}
\renewcommand{\arraystretch}{1.5}
%\addtolength{\tabcolsep}{-1pt}
\centering
\begin{tabular}{*{8}{c}}
\hline
\hline \multirow{2}{*}{Dataset} & \multirowcell{2}{Without privacy protection} &\multicolumn{2}{c}{Modified loss} &\multicolumn{4}{c}{Perturbed ADMM} \\ \cmidrule(lr){3-4}\cmidrule(lr){5-8} \morecmidrules& &{$\epsilon=0.4$, $R=0$} &{$\epsilon=1$, $R=0$}  &{$\epsilon=0.4$, $R=1$} &{$\epsilon=0.4$, $R=9$} &{$\epsilon=1$, $R=1$} &{$\epsilon=1$, $R=9$} \\
\hline
{German}     &75.00 &71.00 &74.00 &69.67 &64.00 &74.33 &67.67  \\
{Image}        &75.56 &70.13 &72.84 &69.33 &63.10 &70.45 &65.50  \\
{Ringnorm}  &77.38 &73.44 &76.82 &73.74 &66.18 &75.77 &70.23  \\
{Banana}     &58.22 &54.33 &56.06 &54.28 &43.11 &55.89 &54.44  \\
{Splice}        &56.60 &46.84 &56.60 &54.94 &46.39 &55.83 &52.50  \\
{Twonorm}   &97.90 &96.59 &97.38 &96.51 &92.28 &97.41 &94.77  \\
{Waveform} &88.93 &84.60 &87.93 &84.07 &80.47 &87.67 &81.73  \\
\hline
\hline
\end{tabular}
\end{table*}

We then study the effects of the key parameters on the performance. In Fig.~\ref{effects_M}, we examine the impact of the noise bound $R$ when the decaying rate $\rho$ is fixed at $0.8$. It is observed that $R$ affects the final empirical risks of the trained classifiers. The larger the noise bound, the greater the gap between the achieved empirical risks and the reference value, which is reconciled with Corollary~\ref{corollary1}. In Fig.~\ref{effects_rho}, we inspect the effect of Gaussian noise decaying rate $\rho$ when $R$ is fixed at $1$. We find that the convergence time is affected by $\rho$. A larger $\rho$ implies that the communicated classifiers are still perturbed by noises with larger variance even after iterating over multiple steps. Thus, more iterations are needed to obtain the same final empirical risk with that of smaller $\rho$. Such a property can be derived from the bound in (\ref{eq_delta_Jt}).

%\begin{figure}
%\begin{center}
%\includegraphics[scale=0.45]{figure/different_rho.pdf}
%\caption{The effects of decaying rate $\rho$.} \label{effects_rho}
%\end{center}
%\end{figure}
Fig. \ref{effects_epsilon} illustrates the variation of final empirical risks when the PPD $\epsilon$ changes. The final empirical risks decrease with larger PPD (weaker privacy guarantee), which implies the tradeoff relation between the privacy protection and the performance. Further, the extra perturbations in Phase~2 lead to larger empirical risks for all the PPDs in the experiments. We also find that when $\epsilon$ is large ($\epsilon>0.6$), the achieved empirical risks are close to the reference value, and do not significantly change. Again, the result is consistent with the analysis of the bound in (\ref{Jp_w_hat}).
%\begin{figure}
%\begin{center}
%
%\caption{The effects of PPD $\epsilon$.}
%\end{center}
%\end{figure}
% {-0.3cm}
\subsection{Classification Accuracy Evaluation}
We use the test data of the seven datasets to evaluate the prediction performance of the trained classifiers, which is shown in Table~\ref{table}. The classification accuracy is defined as the ratio that the labels predicted by the trained classifier match the true labels of test data. For comparison, we present the classification accuracy achieved by general ADMM with the original data. For validation of classification accuracy under the PDML framework, we choose six different sets of parameter configurations to conduct the experiments. The specific configurations can be found in the second row of Table~\ref{table}. We find that lager $\epsilon$ and smaller $R$ will generate better accuracy. According to the theoretical results, the upper bounds for the differences $\Delta J_{\mathcal{P}}(\widetilde{\mathbf{w}}_\mathrm{opt})$ and $ \Delta J_{\mathcal{P}}(\widehat{\mathbf{w}}_\mathrm{opt})$ will decrease with lager $\epsilon$ and smaller $R$, implying better performance of the trained classifiers. Thus, the bound in Theorem \ref{delta_Jt} also provides a guideline to choose appropriate parameters to obtain a prediction model with satisfied classification accuracy.

It is impressive to observe that even under the strongest privacy setting ($\epsilon=0.4$, $R=9$), the proposed framework achieves comparable classification accuracy to the reference precision. We also notice that under the datasets \textit{Banana} and \textit{Splice}, PDML achieves  inferior accuracy in all settings. For a binary classification problem, it is meaningless to obtain a precision of around $50\%$. The reason for the poor accuracy may be that LR is not a suitable classification approach for these two datasets. Overall, the proposed PDML framework achieves competitive classification accuracy on the basis of providing strong privacy protection.

\section{Conclusion} \label{conclusion}
In this paper, we have provided a privacy-preserving ADMM-based distributed machine learning framework. By a local randomization approach, data contributors obtain self-controlled DP protection for the most sensitive labels and the privacy guarantee will not decrease as ADMM iterations proceed. Further, a combined noise-adding method has been designed for perturbing the ADMM algorithm, which simultaneously preserves privacy for users' feature vectors and strengthens protection for the labels. Lastly, the performance of the proposed PDML framework has been analyzed in theory and validated by extensive experiments.

For future investigations, we will study the joint privacy-preserving effects of the local randomization approach and the combined noise-adding method. Moreover, it is interesting while challenging to extend the PDML framework to the non-empirical risk minimization {problems}. When users allocate distinct sensitive levels to different attributes, we are interested in designing a new privacy-aware scheme providing heterogeneous privacy protections for different attributes.

\appendix
\subsection{Proof of Proposition~\ref{privacy_preservation}} \label{proof_p1}
Let $\mathbf{d}'=(\mathbf{x},y')$ be the reported data of a user with arbitrary data sample $\mathbf{d}=(\mathbf{x},y)$ drawn from $\mathcal{P}$. Then we have $\mathbf{d}'=M(\mathbf{d})$. Suppose that the user's data sample has label $y_1=1$, which is denoted by $\mathbf{d}_{1}=(\mathbf{x}, 1)$. By (\ref{randomization}) and (\ref{eq_p}), the probability that the user reports $y_1$ to the server is
\begin{equation}\nonumber
\small
\Pr[y'_1=1\;|\;y_1]=1-p=\frac{e^\epsilon}{1+e^\epsilon}.
\end{equation}
Similarly, if the user's original label is $y_2=-1$, i.e., $\mathbf{d}_{2}=(\mathbf{x}, -1)$, we have
\begin{equation}\nonumber
\small
\Pr[y'_2=1\;|\;y_2]=p=\frac{1}{1+e^\epsilon}.
\end{equation}
Then, we further have the relations as follows:
\begin{equation}\nonumber
\small
\Pr[y'_1=-1\;|\;y_1]=\frac{1}{1+e^\epsilon}, \Pr[y'_2=-1\;|\;y_2]=\frac{e^\epsilon}{1+e^\epsilon}.
\end{equation}
With a slight abuse of notation, we view label ``$-1$" as ``0" below. Note that under this case, the observation set $\mathcal{O}$ in Definition \ref{LDP_def} is the user's reported data $\mathbf{d}'$. Then, for any $\mathbf{d}'$ with feature vector $\mathbf{x}$ and arbitrary label $y'$, we have
\begin{equation}\nonumber
\small
%\begin{split}
\frac{\Pr[\mathbf{d}'\,|\,\mathbf{d}_1]}{\Pr[\mathbf{d}'\,|\,\mathbf{d}_2]}
\leq \max_{y'\in\{1, 0\}} \frac{(1-p)^{y'}\cdot p^{1-y'}}{p^{y'}\cdot(1-p)^{1-y'}} =\frac{1-p}{p}=e^{\epsilon},
%\end{split}
\end{equation}
where we use the relation $p\in(0,\frac{1}{2})$. \qed
\subsection{Proof of Proposition~\ref{unbiased_loss}} \label{proof_l1}
%\begin{proof}
(i) According to (\ref{randomization}), we have
\begin{alignat}{2}
\small
 \quad &\Pr[y'_{i, j}=y_{i, j}\;|\;y_{i, j}] =1-p=\frac{e^\epsilon}{1+e^\epsilon}, \nonumber\\
 \quad &\Pr[y'_{i, j}=-y_{i, j}\;|\;y_{i, j}] =p=\frac{1}{1+e^\epsilon},\nonumber
\end{alignat}
where we have used Proposition~\ref{privacy_preservation}. Then, it follows that
\begin{equation}\label{E_l_hat}
\small
\begin{split}
& \mathbb{E}_{y'_{i, j}}\left[\hat{\ell}(y'_{i, j}, \mathbf{w}_i^{\mathrm{T}}\mathbf{x}_{i, j}, \epsilon)\right] =\Pr[y'_{i, j}=y_{i, j}\;|\;y_{i, j}]\hat{\ell}(y_{i, j}, \mathbf{w}_i^{\mathrm{T}}\mathbf{x}_{i, j}, \epsilon)\\
& \qquad \qquad \qquad \qquad \quad +\Pr[y'_{i, j}=y_{i, j}\;|\,-y_{i, j}]\hat{\ell}(-y_{i, j}, \mathbf{w}_i^{\mathrm{T}}\mathbf{x}_{i, j}, \epsilon) \\
& =\frac{1}{e^\epsilon+1}\left[e^\epsilon\hat{\ell}(y_{i, j}, \mathbf{w}_i^{\mathrm{T}}\mathbf{x}_{i, j}, \epsilon)+\hat{\ell}(-y_{i, j}, \mathbf{w}_i^{\mathrm{T}}\mathbf{x}_{i, j}, \epsilon)\right].
\end{split}
\end{equation}
By (\ref{modified_loss}), we obtain
\begin{equation}\label{epsilon_l_hat}
\small
\begin{split}
& e^\epsilon\hat{\ell}(y_{i, j}, \mathbf{w}_i^{\mathrm{T}}\mathbf{x}_{i, j}, \epsilon)+\hat{\ell}(-y_{i, j}, \mathbf{w}_i^{\mathrm{T}}\mathbf{x}_{i, j}, \epsilon) \\
& = \frac{e^\epsilon}{e^\epsilon-1}\left[e^\epsilon\ell(y_{i, j}, \mathbf{w}_i^{\mathrm{T}}\mathbf{x}_{i, j})-\ell(-y_{i, j}, \mathbf{w}_i^{\mathrm{T}}\mathbf{x}_{i, j})\right] \\
& \quad +\frac{1}{e^\epsilon-1}\left[e^\epsilon\ell(-y_{i, j}, \mathbf{w}_i^{\mathrm{T}}\mathbf{x}_{i, j})-\ell(y_{i, j}, \mathbf{w}_i^{\mathrm{T}}\mathbf{x}_{i, j})\right] \\
& = (e^\epsilon+1)\ell(y_{i, j}, \mathbf{w}_i^{\mathrm{T}}\mathbf{x}_{i, j}).
\end{split}
\end{equation}
Substituting (\ref{epsilon_l_hat}) into (\ref{E_l_hat}), we arrive at
\begin{equation}\nonumber%\label{take_expectation}
\small
\mathbb{E}_{y'_{i, j}}\left[\hat{\ell}(y'_{i, j}, \mathbf{w}_i^{\mathrm{T}}\mathbf{x}_{i, j}, \epsilon)\right] = \ell(y_{i, j}, \mathbf{w}_i^{\mathrm{T}}\mathbf{x}_{i, j}).
\end{equation}

(ii) The derivative of $\hat{\ell}(y'_{i, j}, \mathbf{w}_i^{\mathrm{T}}\mathbf{x}_{i, j}, \epsilon)$ with respect to $\mathbf{w}_i$ is given by
\begin{equation}\nonumber%\label{l_hat_de}
\small
\begin{split}
& \frac{\partial \hat{\ell}(y'_{i, j}, \mathbf{w}_i^{\mathrm{T}}\mathbf{x}_{i, j}, \epsilon)}{\partial \mathbf{w}_i} \\
& =\frac{e^\epsilon}{e^\epsilon-1}\frac{\partial \ell(y_{i, j}, \mathbf{w}_i^{\mathrm{T}}\mathbf{x}_{i, j})}{\partial \mathbf{w}_i} -\frac{1}{e^\epsilon-1}\frac{\partial \ell(-y_{i, j}, \mathbf{w}_i^{\mathrm{T}}\mathbf{x}_{i, j})}{\partial \mathbf{w}_i}.
\end{split}
\end{equation}
Then, we have
\begin{equation}\nonumber%\label{l_hat_bound}
\small
\begin{split}
&  \left|\frac{\partial \hat{\ell}(y'_{i, j}, \mathbf{w}_i^{\mathrm{T}}\mathbf{x}_{i, j}, \epsilon)}{\partial \mathbf{w}_i}\right| \\
& \leq \frac{e^\epsilon}{e^\epsilon-1}\left|\frac{\partial \ell(y_{i, j}, \mathbf{w}_i^{\mathrm{T}}\mathbf{x}_{i, j})}{\partial \mathbf{w}_i}\right|  +\frac{1}{e^\epsilon-1}\left|\frac{\partial \ell(-y_{i, j}, \mathbf{w}_i^{\mathrm{T}}\mathbf{x}_{i, j})}{\partial \mathbf{w}_i}\right| \\
& \leq \frac{e^{\epsilon}+1}{e^{\epsilon}-1}c_2.
\end{split}
\end{equation}
This bound gives the Lipschitz constant of $\hat{\ell}(y'_{i, j}, \mathbf{w}_i^{\mathrm{T}}\mathbf{x}_{i, j}, \epsilon)$. \qed
%\end{proof}
%\vspace{-0.5cm}
\subsection{Proof of Lemma \ref{objective_convex}} \label{proof_l2}
%\begin{proof}
According to Assumption \ref{regularizer_assumption}, $N(\cdot)$ is doubly differentiable. By Taylor's Theorem, we have
\begin{equation}\nonumber
\small
\begin{split}
N(\mathbf{w}_1)& =N(\mathbf{w}_2)+\nabla N(\mathbf{w}_1)^\mathrm{T}(\mathbf{w}_2-\mathbf{w}_1) \\
& +\frac{1}{2}(\mathbf{w}_2-\mathbf{w}_1)^{\mathrm{T}}\nabla^2 N(\mathbf{w}_1)(\mathbf{w}_2-\mathbf{w}_1)+o\left(\|\mathbf{w}_2-\mathbf{w}_1\|_2^2\right),
\end{split}
\end{equation}
where $\nabla N(\cdot)$ and $\nabla^2 N(\cdot)$ denote the gradient and the second-order gradient, respectively. Due to (\ref{strongly_convex}), we derive
\begin{equation}\nonumber
\small
(\mathbf{w}_2-\mathbf{w}_1)^{\mathrm{T}}\nabla^2 N(\mathbf{w}_1)(\mathbf{w}_2-\mathbf{w}_1)\geq \kappa\|\mathbf{w}_2-\mathbf{w}_1\|_2^2,
\end{equation}
which implies $\nabla^2 N(\mathbf{w})\geq \kappa$. For $\forall i$, let $f(\mathbf{w}_i):=\widehat{J}_i(\mathbf{w}_i)-\frac{a\kappa}{2n}\|\mathbf{w}_i\|_2^2$, and then we have
\begin{equation}\nonumber
\small
\begin{split}
\frac{\partial^2 f(\mathbf{w}_i)}{\partial \mathbf{w}_i^2} & = \frac{1}{m_i}\sum_{j=1}^{m_i} \frac{\partial^2 \hat{\ell}(y'_{i, j}, \mathbf{w}_i^{\mathrm{T}}\mathbf{x}_{i, j}, \epsilon)}{\partial \mathbf{w}_i^2} + \frac{a}{n}\nabla^2 N(\mathbf{w}_i)-\frac{a\kappa}{n} \\
& \geq \frac{1}{m_i (e^{\epsilon}-1)}\sum_{j=1}^{m_i} \left[e^{\epsilon}\frac{\partial^2 \ell(y'_{i, j}, \cdot)}{\partial \mathbf{w}_i^2}-\frac{\partial^2 \ell(-y'_{i, j}, \cdot)}{\partial \mathbf{w}_i^2}\right] \\
& = \frac{1}{m_i}\sum_{j=1}^{m_i} \frac{\partial^2 \ell(y'_{i, j}, \mathbf{w}_i^{\mathrm{T}}\mathbf{x}_{i, j})}{\partial \mathbf{w}_i^2}\geq 0,
\end{split}
\end{equation}
where we have used Assumption \ref{loss_assumption}. This relation also implies that $f(\mathbf{w}_i)$ is convex. Then, we obtain $\forall \mathbf{w}_1, \mathbf{w}_2\in\mathcal{W}$, $f(\mathbf{w}_1)\geq f(\mathbf{w}_2)+\nabla f(\mathbf{w}_2)^{\mathrm{T}}(\mathbf{w}_1-\mathbf{w}_2)$. It follows that
\begin{equation}\nonumber
\small
\begin{split}
\widehat{J}_i(\mathbf{w}_1)-\frac{a\kappa}{2n}\|\mathbf{w}_1\|_2^2\geq & \widehat{J}_i(\mathbf{w}_2)+\frac{a\kappa}{2n}\|\mathbf{w}_2\|_2^2 \\
& +\nabla \widehat{J}_i(\mathbf{w}_2)^{\mathrm{T}}(\mathbf{w}_1-\mathbf{w}_2)-\frac{a\kappa}{n}\mathbf{w}_2^{\mathrm{T}} \mathbf{w}_1.
\end{split}
\end{equation}
Rearrange the above equation so that
\begin{equation}\nonumber
\small
\widehat{J}_i(\mathbf{w}_1)-\widehat{J}_i(\mathbf{w}_2)\geq \nabla \widehat{J}_i(\mathbf{w}_2)^{\mathrm{T}}(\mathbf{w}_1-\mathbf{w}_2) + \frac{a\kappa}{2n}\|\mathbf{w}_1-\mathbf{w}_2\|_2^2,
\end{equation}
which indicates $\widehat{J}_i(\mathbf{w}_i)$ is $\frac{a\kappa}{n}$-strongly convex. Since $\widehat{J}(\{\mathbf{w}_i\}_{i\in\mathcal{S}})=\sum_{i=1}^{n} \widehat{J}_i(\mathbf{w}_i)$, it follows that $\widehat{J}(\{\mathbf{w}_i\}_{i\in\mathcal{S}})$ is $a\kappa$-strongly convex. \qed
%\end{proof}
\subsection{Proof of Lemma \ref{J_tildle_convex}} \label{proof_l4}
%\begin{proof}
The strongly convex property of $\widetilde{J}(\{\mathbf{w}_i\}_{i\in\mathcal{S}})$ can be proved directly from Lemma \ref{objective_convex}. For the Lipschitz continuous gradient, we consider the compact form of classifiers, as $\mathbf{w}=[\mathbf{w}_1^\mathrm{T} \cdots \mathbf{w}_n^\mathrm{T}]^\mathrm{T}$. We have $\widetilde{J}(\mathbf{w})=\widetilde{J}(\{\mathbf{w}_i\}_{i\in\mathcal{S}})$. The second derivative of $\widetilde{J}(\mathbf{w})$ with respect to $\mathbf{w}$ is given by
\begin{equation}\nonumber
\small
\begin{split}
\frac{\partial^2 \widetilde{J}(\mathbf{w})}{\partial \mathbf{w}^2}& = \sum_{i=1}^{n} \frac{\partial^2 \widetilde{J}_i(\mathbf{w}_i)}{\partial \mathbf{w}_i^2} \\
& = \sum_{i=1}^{n}\left[\frac{1}{m_i}\sum_{j=1}^{m_i} \frac{\partial^2 \hat{\ell}(y'_{i, j}, \cdot)}{\partial \mathbf{w}_i^2} + \frac{a}{n}\nabla^2 N(\mathbf{w}_i)\right].
\end{split}
\end{equation}
For $\frac{\partial^2 \hat{\ell}(y'_{i, j}, \mathbf{w}_i^{\mathrm{T}}\mathbf{x}_{i, j}, \epsilon)}{\partial \mathbf{w}_i^2}$, we have
\begin{equation}\nonumber
\small
\begin{split}
\left|\frac{\partial^2 \hat{\ell}(y'_{i, j}, \cdot)}{\partial \mathbf{w}_i^2}\right|& = \frac{1}{e^{\epsilon}-1} \left|e^{\epsilon}\frac{\partial^2 \ell(y'_{i, j}, \cdot)}{\partial \mathbf{w}_i^2}-\frac{\partial^2 \ell(-y'_{i, j}, \cdot)}{\partial \mathbf{w}_i^2}\right| \\
& = \left|\frac{\partial^2 \ell(y'_{i, j}, \cdot)}{\partial \mathbf{w}_i^2}\right|\leq c_3.
\end{split}
\end{equation}
Due to $\|\nabla^2 N(\cdot)\|_2\leq \varrho$, we derive
\begin{equation}\nonumber
\small
\left|\frac{\partial^2 \widetilde{J}(\mathbf{w})}{\partial \mathbf{w}^2} \right|\leq \sum_{i=1}^{n} \left[c_3 + \frac{a}{n} \varrho\right] = nc_3 + a\varrho.
\end{equation}
This also gives the Lipschitz continuous gradient of $\widetilde{J}(\mathbf{w})$. \qed
%\end{proof}
%\appendices
%% you can choose not to have a title for an appendix
%% if you want by leaving the argument blank
%\section{}

%\section*{Acknowledgment}

%\ifCLASSOPTIONcaptionsoff
%  \newpage
%\fi
{\small %\scriptsize
\bibliographystyle{IEEEtran}        % Include this if you use bibtex
\bibliography{root}   }        % and a bib file to produce the
\end{document}